\documentclass[pdflatex,sn-mathphys-num]{sn-jnl}

\usepackage{graphicx}%
\usepackage{multirow}%
\usepackage{mathtools}
\usepackage{amsmath,amssymb,amsfonts}%
\usepackage{amsthm}%
\usepackage{mathrsfs}%
\usepackage[title]{appendix}%
\usepackage{xcolor}%
\usepackage{textcomp}%
\usepackage{caption}
\usepackage{manyfoot}%
\usepackage{mathrsfs}
\usepackage{booktabs}%
\usepackage{algorithm}%
\usepackage{comment}
\usepackage{tikz}
\usetikzlibrary{positioning}
\usepackage{algorithmicx}%
\usepackage{algpseudocode}%
\usepackage{listings}%

\theoremstyle{thmstyleone}%
\newtheorem{theorem}{Theorem}
\newtheorem{proposition}[theorem]{Proposition}%

\theoremstyle{thmstyletwo}%
\newtheorem{lemma}{Lemma}
\theoremstyle{thmstylethree}%
\newtheorem{definition}{Definition}%

\begin{document}

\title[SUNLayer]{SUNLayer: Stable denoising with generative networks}


\author[1]{\fnm{Ruhui} \sur{Jin}} \email{rjin18@jh.edu}

\author[2]{\fnm{Dustin~G.} \sur{Mixon}}\email{mixon.23@osu.edu}

\author*[1]{\fnm{Soledad} \sur{Villar}}\email{svillar3@jhu.edu}
\equalcont{The author list is presented in alphabetical order.}

\affil[1]{\orgdiv{Department of Applied Mathematics and Statistics}, \orgname{Johns Hopkins University}, \orgaddress{\city{Baltimore}, \state{Maryland}}}

\affil[2]{\orgdiv{Department of Mathematics}, \orgname{The Ohio State University}, \orgaddress{\city{Columbus}, \state{Ohio}}}


\abstract{
Deep neural networks are often used to implement powerful generative models for real-world data. Notable applications include image denoising, as well as other classical inverse problems like compressed sensing and super-resolution. 
To provide a rigorous but simplified analysis of generative models, in this work, we introduce an elegant theoretical framework based on spherical harmonics, namely \textbf{SUNLayer}. Our theoretical framework identifies explicit conditions on activation functions that guarantee denoising under local optimization. Numerical experiments examine the theoretical properties on commonly used activation functions and demonstrate their stable denoising performance. }


\maketitle



\section{Introduction}



Deep neural networks are widely applied to implement generative models for real-world data (e.g. diffusion-based models \cite{song2020denoising}, variational autoencoders \cite{kingma2013auto}, generative adversarial networks \cite{goodfellow2014generative}). Their success is particularly evident in applications involving natural images (see, for instance, \cite{nguyen2016synthesizing}). These generative priors have been successfully applied to classical inverse problems in signal processing, including super-resolution~\cite{johnson2016perceptual} and compressed sensing~\cite{bora2017compressed}. For example, \cite{bora2017compressed} demonstrates numerically that generative priors can solve compressed sensing problems with roughly ten times fewer measurements than required by traditional compressed sensing theory.
Follow-up work by \cite{hand2017global} analyzed optimization methods, specifically, empirical risk minimization, in compressed sensing by fitting the data with multi-layer neural networks having random weights and ReLU activation functions.


The aim of this paper is to develop a theoretical analysis for neural network method's behavior in signal denoising, a classical inverse problem task. While deep neural networks have demonstrated strong empirical performance in applications such as image denoising and inpainting \cite{xie2012image}, existing theoretical results are still scarce. In this work, we focus on the high-noise regime of denoising, where traditional, non–learning-based methods often fail to achieve satisfactory performance. 

To this end, we introduce \textbf{SUNLayer}, a generative network with a simplified yet expressive structure, where signals are generated through compositions of a linear map and nonlinear activation functions. Leveraging tools from spherical harmonic analysis \cite{morimoto1998analytic}, we derive explicit conditions on the activation functions that ensure stable reconstruction by local methods, in terms of the optimization landscape. Our analysis provides a theoretical foundation for the criteria of ``good" nonlinearities in generative network denoisers and clarifies how network basis contributes to the stable recovery in inverse problems.

Our theory suggests a systematic way to predict how well a model with a certain nonlinearity would perform at denoising. We compute these metrics for a set of well-known nonlinearities.
We compare the denoising properties of models with the different nonlinearities in toy inverse problems, obtaining results consistent with the theory.

\subsection{Related works}


Classical denoising methods formulate signal recovery as an inverse problem, reconstructing a clean signal $x$ from a noisy observation $y$. Prominent examples include compressive sensing \cite{metzler2016denoising}, phase retrieval \cite{candes2015phase}, and total variation regularization \cite{rudin1992nonlinear}. They rely on explicit, mathematically prescribed priors and regularization to exploit data structural assumptions such as sparsity, low-rankness, or piecewise smoothness. These approaches can yield strong recovery guarantees even with limited data. However, their reliance on carefully designed priors restricts their ability to capture the rich diversity of natural signals, motivating the development of data-driven alternatives.

In contrast, deep generative models learn signal priors directly from data to capture the distribution of clean signals, typically assuming that they are near a low-dimensional manifold. This paradigm has become central to modern denoising, enabling expressive modeling of complex signals such as natural images, medical and scientific measurements. Representative frameworks include GANs \cite{goodfellow2014generative}, Variational Autoencoders (VAEs) \cite{kingma2013auto}, and more recently, diffusion models \cite{song2020denoising}. Despite their strong empirical performance, deep generative denoisers pose significant theoretical challenges, including nonconvex optimization landscapes and fragile stability guarantees. 

Substantial efforts have aimed to understand the theoretical guarantees of deep generative priors. One of the first works in this direction is \cite{bora2017compressed}, which introduced generative models as structured priors for compressive sensing. Building on this idea, \cite{hand2017global} analyzed global recovery via empirical risk minimization under random measurements, with sample complexity scaling with the latent dimension. 
Subsequent works investigated algorithmic guarantees for solving the resulting nonconvex problems, with \cite{shah2018solving}, \cite{gomez2019fast}, and \cite{huang2021provably} establishing provable convergence results for first-order methods such as PGD, SGD, and ADMM under suitable assumptions on the generator.

More closely, 
\cite{tripathi2018correction} analyzes denoising via latent-space optimization on GAN manifolds, providing recovery guarantees under low-dimensional generative assumptions. In contrast, the proposed SUNLayer framework studies denoising from a local geometric perspective. By using spherical harmonics analytical tool, SUNLayer characterizes the local stability, offering theoretical knowledge into what kind of generative networks (activation functions) act as stable denoisers. Another related work \cite{bach2017breaking} studies a basic neural architecture and its statistical learning power, emphasizing how certain activation functions enable a convex optimization perspective. While not directly concerned with generative modeling, \cite{bach2017breaking} uses a similar simplification of neural networks to analyze deep learning.

More broadly, state-of-the-art denoising methods increasingly rely on diffusion- and score-based generative models, which learn the score of the data distribution via denoising objectives and perform noise removal through stochastic sampling dynamics. We view this line of work as adjacent and complementary to the present study.

\subsection{Main contributions}



This work highlights the following contributions. We introduce SUNLayer, a prototype model with spherical uniform neural network layers for the theoretical analysis of generative neural networks (Section~\ref{sec:sunlayer}). The model is designed to capture the essential structure of generative networks while enabling the use of spherical harmonics techniques and remaining analytically tractable. Building on this framework, we establish performance guarantees for denoising. In particular, we show that the associated nonconvex reconstruction objective achieves near-optimal critical points, depending on noise level and network components.  We provide explicit conditions on the activation function for provable reconstruction (Section~\ref{sec:theory}). Finally, the theoretical framework developed in this paper can extend beyond denoising and offers a basic approach for gaining insights into the behavior of deep generative models applied to general inverse problems; see Section~\ref{sec:discussion} for further discussions.

\section{SUNLayer: a neural network model} \label{sec:sunlayer}  \label{sec:finite_dim}



Let $x\in S^{n}$ be an input signal, we consider the linear map $x\mapsto f_x \in \mathscr{L}^2(S^{n})$ where $f_x(y)=x\cdot y$ the inner product in $\mathbb{R}^{n+1}$ between $x$ and $y$. 
Let $\theta:\mathbb R\to \mathbb R$ be an activation function. We define one layer of the SUNLayer neural network to be 
\begin{eqnarray}
L_{n}: S^{n} &\to& \mathscr{L}^2(S^{n})  \label{eq.layer}\\ 
L_{n}(x)&=&\theta\circ f_x \nonumber.
\end{eqnarray}

Note that if instead of the linear map $f_x$ we had considered, as one usually does in neural networks, a matrix $M\in \mathbb R^{t\times n+1}$, then the analogous to $L_{n}(x)$ is essentially $\theta(Mx)$ that can be seen as a function defined in the rows of $M$ as $L(x): \{1,\ldots t\}\to \mathbb R$ as $m_i\mapsto \theta(x \cdot m_i)$. The SUNLayer model is heuristically generalizing the linear step to a continuum of possible rows. 


We are interested in the case where $L_{n}(S^{n})\subset A \subset \mathscr{L}^2(S^{n})$ where $A$ is a finite dimensional subspace of $\mathscr{L}^2(S^{n})$ (and therefore locally compact). The finite dimensionality will allow us to compose several layers of the SUNLayer model. For all $x\in S^{n}$, we have that $\frac{L_{n}(x)}{\|L_{n}(x)\|} \in S^{n'}$ with $\operatorname{dim}(A)=n'+1$. A very simple observation (see proof of Lemma \ref{lem.max}) shows that $\|L_{n}(x)\|= \|L_{n}(x')\| = c_{n,\theta}$ for all $x,x'\in S^{n}$ where $c_{n,\theta}$ is a constant that depends on the activation function $\theta$ and on the dimension $n$ of the domain. Therefore the normalization step (which a priori may have resembled practice standards like batch normalization~(\cite{ioffe2015batch})) amounts to simple
rescaling, and furthermore, we even have $c_{n,\theta}=1$ when $\theta$ is
scaled appropriately (see Lemma~\ref{lemma.theta}). 

We then conclude that  $L_{n'}\circ L_{n}:S^{n} \to \mathscr{L}^2(S^{n'})$ is well defined as long that $A$ is finite dimensional. In Section~\ref{sec:theory} we observe that a necessary and sufficient condition for $A$ to be finite dimensional is that $\theta$ is a polynomial.

\subsection{Denoising} \label{sec:denoising}
Let us assume we have a generative model $G: S^{n} \to \mathbb R^N$ that given a parameter $x\in  S^{n}$ produces $G(x)$, an element of a target space (for instance an image)\footnote{The generative model could have been produced for instance with a generative adversarial network (GAN) trained with a large set of images or more generally structured dataset (that comes from an unknown latent distribution). The GAN consists of two neural networks, one known as the generator, which aims to construct new data plausible to be coming from the latent distribution of the training set, and the other is the discriminator which aims to distinguish between instances from the true dataset and the candidates produced by the generator. Both networks get trained against each other. \newline
After training the generator produces a neural network with several layers. We assume the parameter is space is normalized, so the generator finds a generative model $G:  S^{n} \to \mathbb R^N$ where $n \ll N$. For all $x$ we have that $G(x)$ is an element in the target space (for instance, an image) and $x$ is the vector of parameters that generates it.}.
The question we aim to answer is when is it possible to \emph{denoise} an element $y \in \mathbb R^N$ to the closest element in the image of $G$ 
by using local methods like gradient descent. Figure~\ref{fig.denoising} shows an example of the phenomenon we aim to explain.  

We assume our generative model is the composition of layers from the SUNLayer model defined in \eqref{eq.layer}. We solve the denoising problem one layer at a time. 
Fix $x^\sharp \in S^{n}$. Given $y=\theta\circ f_{x^\sharp} + \eta$  for some $\theta:\mathbb R \to \mathbb R$ and noise $\eta\in \mathscr{L}^2(S^{n})$, then denoising for one SUNLayer corresponds with the least squares problem
\begin{equation}
\label{lsp}
\min_{x\in S^{n}} \|\theta \circ f_{x} - y\|^2_{\mathscr{L}^2(S^{n})}.
\end{equation}
There exists at least one minimizer for \eqref{lsp} due to compactness. 

\begin{figure}
\begin{center}
\begin{minipage}{\textwidth}
\includegraphics[width=0.135\textwidth, height=0.135\textwidth]{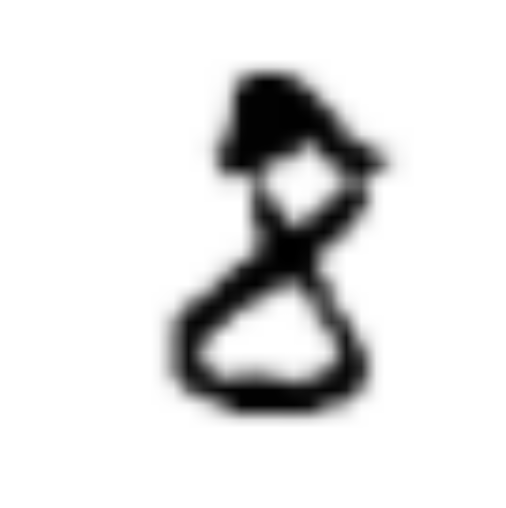}
\includegraphics[width=0.135\textwidth, height=0.135\textwidth]{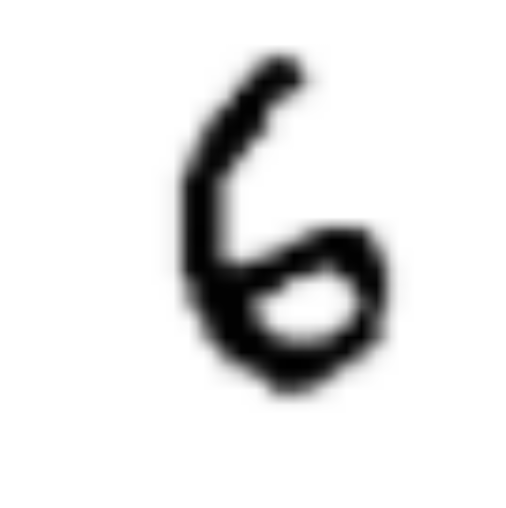}
\includegraphics[width=0.135\textwidth, height=0.135\textwidth]{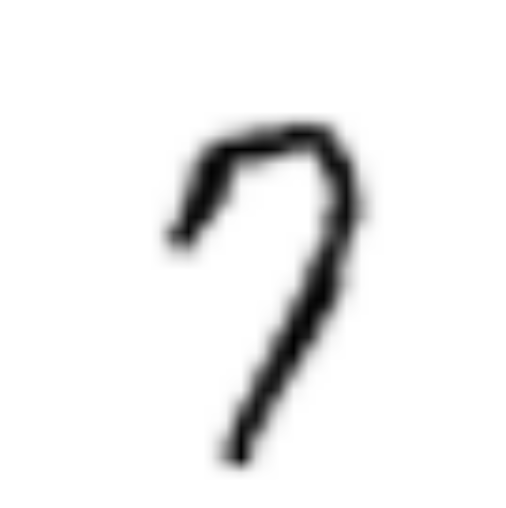}
\includegraphics[width=0.135\textwidth, height=0.135\textwidth]{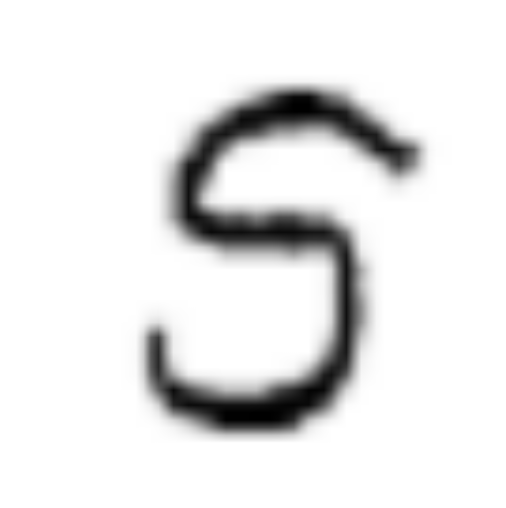}
\includegraphics[width=0.135\textwidth, height=0.135\textwidth]{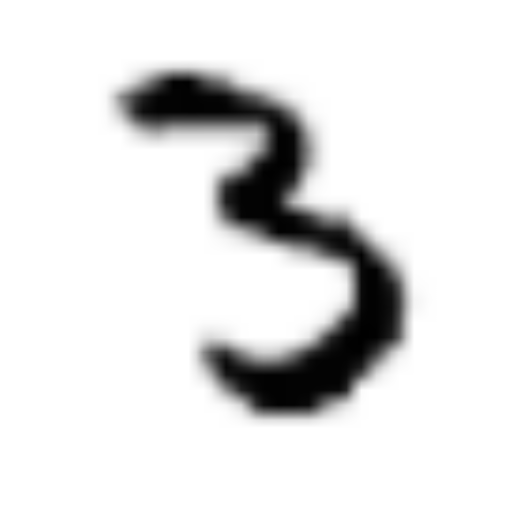}
\includegraphics[width=0.135\textwidth, height=0.135\textwidth]{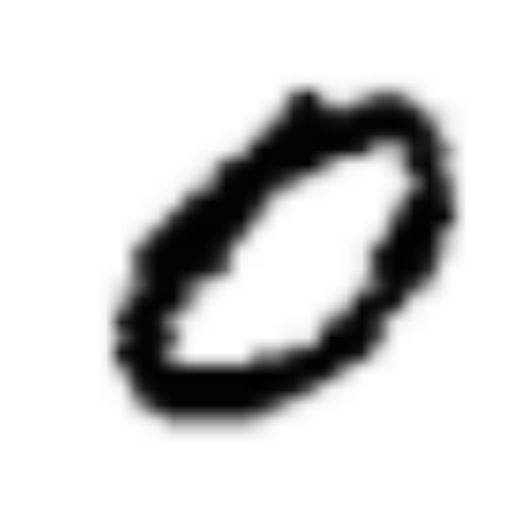}
\includegraphics[width=0.135\textwidth, height=0.135\textwidth]{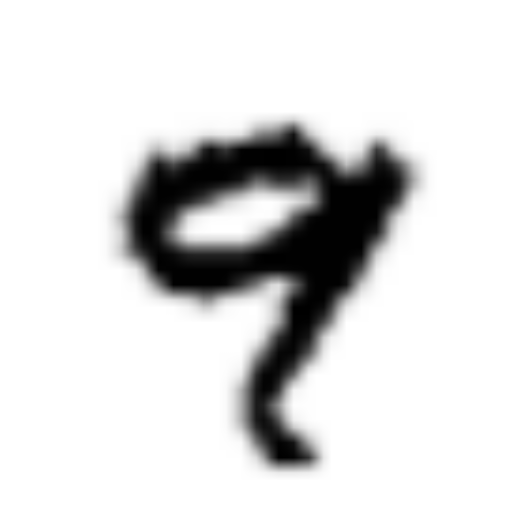}
\newline
\includegraphics[width=0.135\textwidth, height=0.135\textwidth]{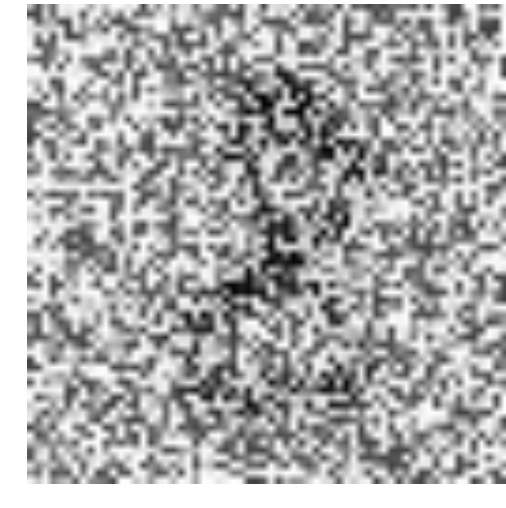}
\includegraphics[width=0.135\textwidth, height=0.135\textwidth]{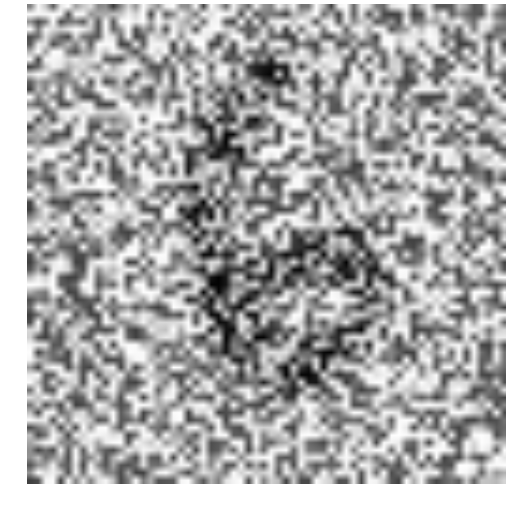}
\includegraphics[width=0.135\textwidth, height=0.135\textwidth]{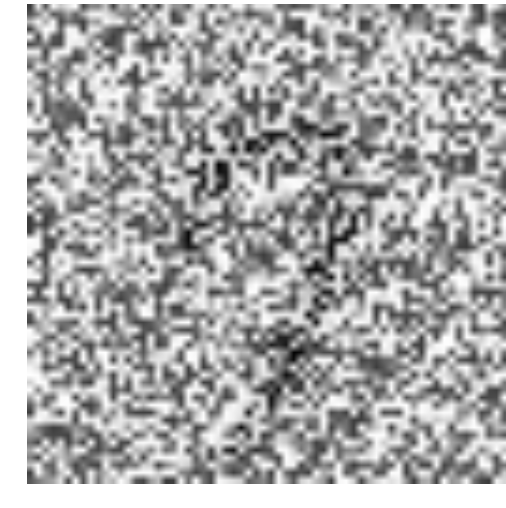}
\includegraphics[width=0.135\textwidth, height=0.135\textwidth]{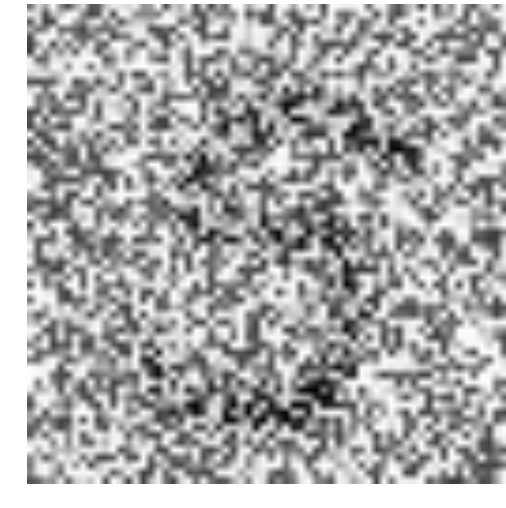}
\includegraphics[width=0.135\textwidth, height=0.135\textwidth]{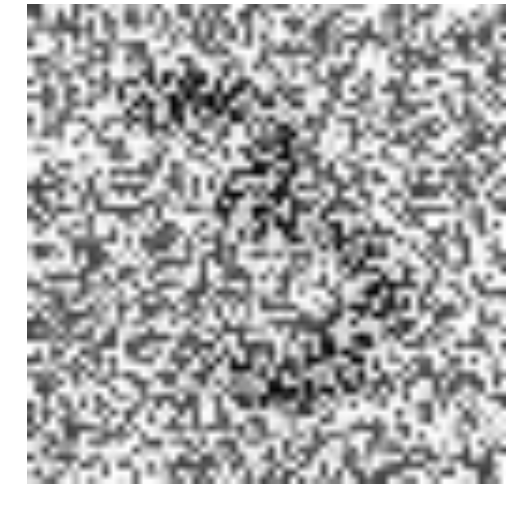}
\includegraphics[width=0.135\textwidth, height=0.135\textwidth]{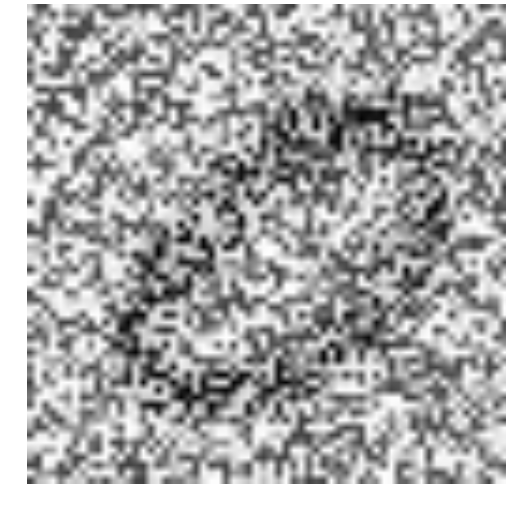}
\includegraphics[width=0.135\textwidth, height=0.135\textwidth]{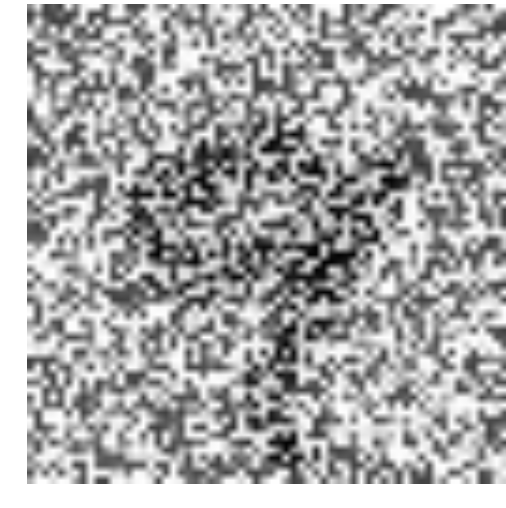}
\newline
\includegraphics[width=0.135\textwidth, height=0.135\textwidth]{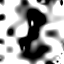}
\includegraphics[width=0.135\textwidth, height=0.135\textwidth]{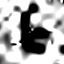}
\includegraphics[width=0.135\textwidth, height=0.135\textwidth]{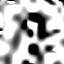}
\includegraphics[width=0.135\textwidth, height=0.135\textwidth]{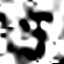}
\includegraphics[width=0.135\textwidth, height=0.135\textwidth]{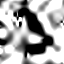}
\includegraphics[width=0.135\textwidth, height=0.135\textwidth]{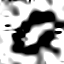}
\includegraphics[width=0.135\textwidth, height=0.135\textwidth]{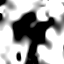}
\newline
\includegraphics[width=0.135\textwidth, height=0.135\textwidth]{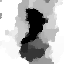}
\includegraphics[width=0.135\textwidth, height=0.135\textwidth]{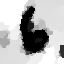}
\includegraphics[width=0.135\textwidth, height=0.135\textwidth]{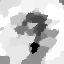}
\includegraphics[width=0.135\textwidth, height=0.135\textwidth]{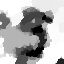}
\includegraphics[width=0.135\textwidth, height=0.135\textwidth]{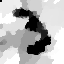}
\includegraphics[width=0.135\textwidth, height=0.135\textwidth]{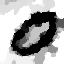}
\includegraphics[width=0.135\textwidth, height=0.135\textwidth]{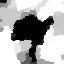}
\newline
\includegraphics[width=0.135\textwidth, height=0.135\textwidth]{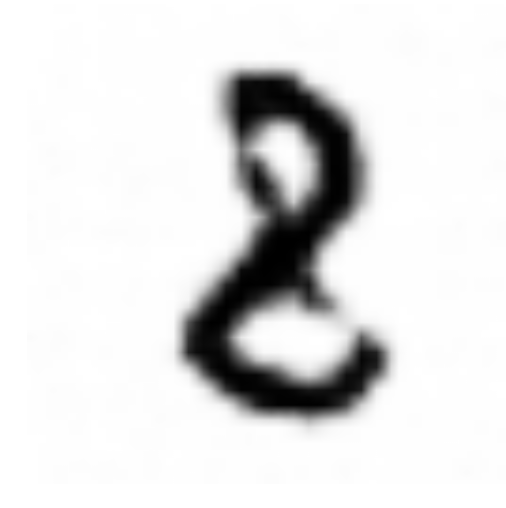}
\includegraphics[width=0.135\textwidth, height=0.135\textwidth]{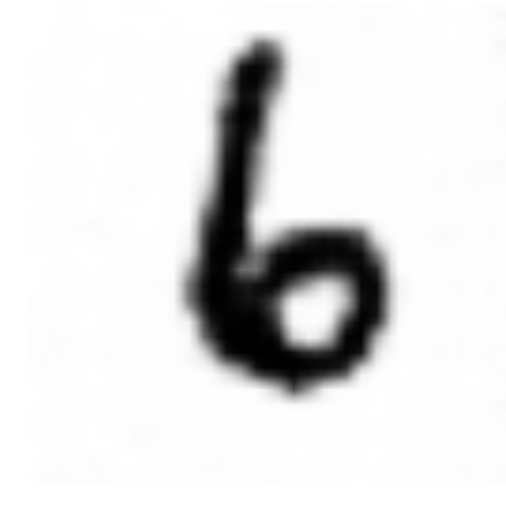}
\includegraphics[width=0.135\textwidth, height=0.135\textwidth]{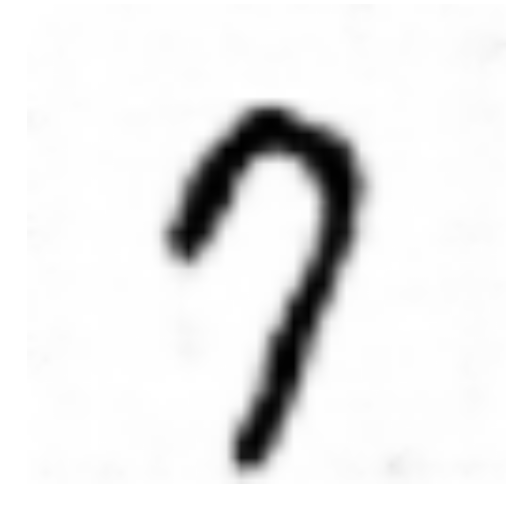}
\includegraphics[width=0.135\textwidth, height=0.135\textwidth]{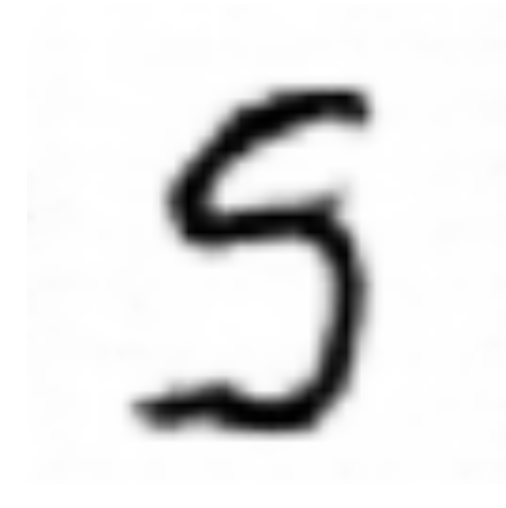}
\includegraphics[width=0.135\textwidth, height=0.135\textwidth]{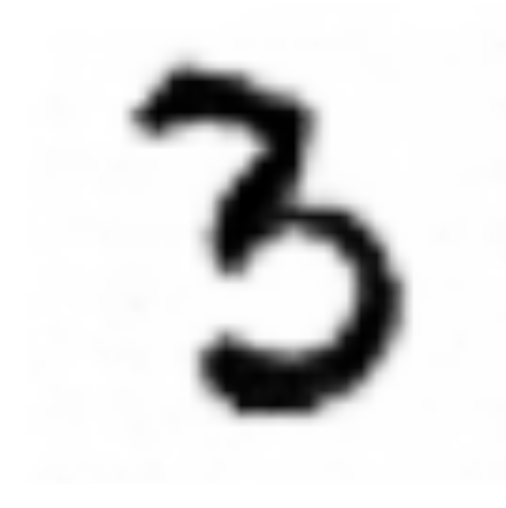}
\includegraphics[width=0.135\textwidth, height=0.135\textwidth]{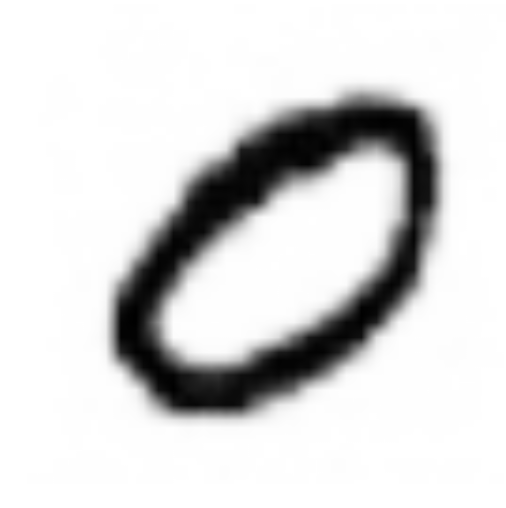}
\includegraphics[width=0.135\textwidth, height=0.135\textwidth]{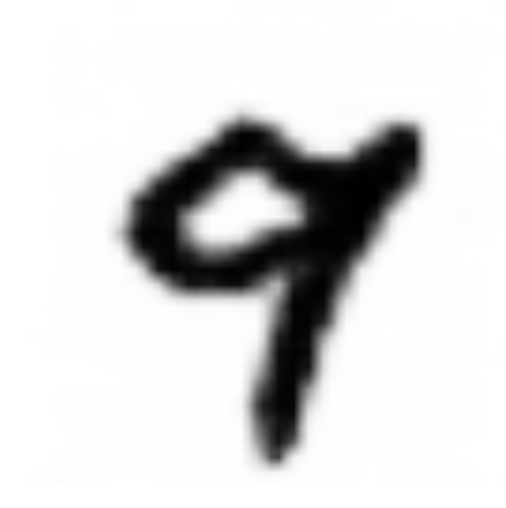}
\end{minipage}
\end{center}
\caption{Denoising with generative priors}{\label{fig.denoising} \small \textbf{ (First line)} Digits from the MNIST test set~(\cite{lecun1998mnist}). \textbf{(Second line)} random noise is added to the digits. \textbf{(Third line)} Denoising of images by shrinkage in wavelet domain~(\cite{donoho1994ideal}). \textbf{(Fourth line)} Denoising by minimizing total variation~(\cite{rudin1992nonlinear}). \textbf{(Fifth line)} We train a GAN using the training set of MNIST to obtain a generative model $G$. We denoise by finding the closest element in the image of $G$ using stochastic grading descent.}
\end{figure}

\section{Preliminaries: spherical harmonics}
\label{app.harmonics}
To analyze denoising under the SUNLayer model, we leverage ideas from spherical harmonics.
In this section we summarize some classical results about spherical harmonics that can be found on Chapter 2 of~\cite{morimoto1998analytic}, focusing on theorems and definitions we use in this paper. We refer the reader to \cite{morimoto1998analytic} for a comprehensive review.

Let $\mathcal P_k(\mathbb R^{n+1})$ the space of homogeneous polynomials of degree $k$ in $n+1$ variables (we could have also considered real or complex coefficients but real is enough for the scope of this paper). 
\begin{definition}[Spherical harmonics] The Laplacian is the differential operator defined as
$$\Delta_x=\frac{\partial^2}{\partial x_1^2} + \ldots + \frac{\partial^2}{\partial x_{n+1}^2},$$ 
and the space of spherical harmonics is defined as:
\begin{equation}
\mathcal H_k(S^{n}) = \{H_k\in \mathcal P_k(S^{n}): \Delta H_k=0 \} \subset  \mathscr{L}^2(S^{n}).
\end{equation}
In other words, $\mathcal H_k(S^{n})$ is the restriction of the polynomials with Laplacian 0 to $S^{n}$.
\end{definition}


\begin{proposition}
$\mathcal H_k(S^{n})$ is a finite dimensional space and 
\begin{equation} \label{l2.eq} \mathscr{L}^2(S^{n})= \oplus_{k=0}^\infty \mathcal H_k(S^{n}).
\end{equation}
\end{proposition}
In the sequel, we let $\alpha_{n,k}$ denote the dimension of $\mathcal{H}_k(S^{n})$.

\begin{definition}
For fixed $k$ and $n$ let $\{Y_k^{1}, \ldots Y_{k}^{\alpha_{n,k}}\}$ an orthonormal basis of $\mathcal H_k(S^{n})$. Define the bilinear form
$$F_{k}(\sigma, \tau) = \sum_{i=1}^{\alpha_{n,k}} Y_k^i(\sigma)\, \overline{Y_k^i(\tau)}.$$
\end{definition}
A simple computation shows that $F_k$ is independent of the choice of the orthonormal basis. 
The bilinear forms $F_k(\cdot, \cdot)$ will be very useful in the analysis of the SUNLayer model. Some of their relevant properties are summarized in the following lemma.


\begin{proposition} The following statements hold.
\begin{enumerate}
\item Reproducing property: $\langle H, F_k(\sigma, \cdot)\rangle  = H(\sigma)$ for all $H\in \mathcal H_k(S^{n})$.
\item Zonal property: there exists $\varphi_{n,k}:\mathbb R\to \mathbb R$ so that $\langle F_k(\sigma_1, \cdot), F_{k}(\sigma_2, \cdot) \rangle = F_k (\sigma_1, \sigma_2) = \varphi_{n,k}(\sigma_1 \cdot \sigma_2)$. In particular $F_k (\sigma_1, \sigma_2)$ only depends on $\sigma_1 \cdot \sigma_2$.
\item The function $\varphi_{n,k}:\mathbb R\to \mathbb R$ is the Gegenbauer polynomial of degree $k$ and dimension $n+1$. The set $\{\varphi_{n,k}\}_{k=0}^\infty$ is an orthogonal basis of polynomials over $[-1,1]$ with respect to the measure 
\begin{equation}d\mu_n = (1-t^2)^{(n-2)/2}dt
\label{measure}
\end{equation} (here $dt$ is the standard Borel measure in $\mathbb R$). 
 Note that this is not a standard normalization for the Gegenbauer polynomials but we use it to simplify the results of this paper. In fact Chapter 2 of~\cite{morimoto1998analytic} considers the Legendre polynomials to be $P_{n,k}(t)=\frac{\operatorname{vol}(S^{n})}{\alpha_{n,k}}\varphi_{n,k}(t)$ (the term $\operatorname{vol}(S^{n})$ is the $n$-dimensional volume of the sphere and it does not show up in Morimoto's analysis since he uses the normalized measure in the spheres). In Chapter 5 Morimoto considers the Gegenbauer polynomials as a generalization of the Legendre polynomials where $n>0$ can be any real number, with a different normalization. 
 
\item The discussion in pages 26--27 of~\cite{morimoto1998analytic} shows that $P_{n,k}(1)=1$. This together with the facts $\alpha_{n,0}=1, \alpha_{n,1}=n+1,$ $$ \alpha_{n,k}=\binom{n+k}{k}-\binom{n+k-2}{k-2} = \frac{(2k+n-1)(k+n-2)!}{k!(n-1)!} = O(k^{n-1}) \quad \text{ for $k\geq 2$, }$$
and $\displaystyle \operatorname{vol}(S^{n})= \frac{2\pi^{(n+1)/2}}{\Gamma((n+1)/2)}$ allow us to identify the correct normalization for the Gegenbauer polynomials. 
\item Using that $\varphi_{n,k}(t)= \frac{\alpha_{n,k}}{\operatorname{vol}(S^n)} P_{n,k}(t)$ and Theorems 2.29 and 2.34 of \cite{morimoto1998analytic} one obtains the following identities:
\begin{align}
&\|\varphi_{n,k}(t)\|_\infty = \varphi_{n,k}(1) = \frac{\alpha_{n,k}}{\operatorname{vol}(S^n)} \label{bound.inf}, \\
&\|\varphi_{n,k}(t)\|_{\mathscr{L}^2(\mu_n)}^2 = \int_{-1}^1 \varphi_{n,k}(t)^2 (1-t^2)^{(n-2)/2}dt = \frac{\alpha_{n,k}}{\operatorname{vol}(S^n)\operatorname{vol}(S^{n-1})}. \label{bound.norm}
\end{align}
\item Using (5.1) and (5.3) of~\cite{morimoto1998analytic} (pages 97--98) one can express a relationship between $\varphi_{n,k}(t)$ and its derivative $\varphi_{n,k}'(t):=\frac{d}{dt}\varphi_{n,k}(t)$, namely
\begin{equation}
\varphi_{n,k}'(t)= \frac{(n+1)\operatorname{vol}(S^n)}{\operatorname{vol}(S^{n+2})} \varphi_{n+2, k-1}(t). \label{bound.derivative}
\end{equation}
\item Let $h\in \mathscr{L}^2(S^{n})$ a $C^{2r}$ function, then one can decompose $h$ in the spherical harmonics as $h(\tau)=\sum_{k=0}^\infty h_k(\tau)$ where $h_k\in\mathcal H_k(S^n)$. Theorem 2.45 of~\cite{morimoto1998analytic} in particular shows that for all $k\geq 0$ one has
\begin{equation}
\label{c2.bound}
k^{2r}\|h_k(\tau)\|_{\mathscr{L}^2(S^{n})} \leq \|(\Delta_{S^n})^{r} h\|_{\mathscr{L}^2(S^{n})}
\end{equation}
where $\Delta_{S^n}$ is the spherical Laplacian. In particular, if there exists an axis under which $h$ is rotationally invariant (i.e. $h(\tau) = \theta(\omega\cdot \tau)$ for some fixed $\omega$ and some $\theta:[-1,1]\to \mathbb R$) then if $\omega\cdot \tau =t$
\begin{equation}\Delta_{S^n}(h) = \theta''(t) (1- t^2) - nt \theta'(t) 
\label{eq.laplacian}
\end{equation}
(see for instance (2.9)). 

\end{enumerate}
\end{proposition}

Note that $F_k(\sigma, \cdot) \in \mathcal H_k(S^{n})$ for all $\sigma \in S^{n}$, thus $\operatorname{span}(\{F_k(\sigma,\cdot)\}_{\sigma\in S^{n}})\subseteq \mathcal H_k(S^{n})$. The reproducing property says that for all $H\in \mathcal H_k(S^{n})$ 
\begin{equation} \label{eq.reproducing}
\langle H, F_k(\sigma, \cdot)\rangle = H(\sigma).
\end{equation}
Observe that for all $H\neq 0$ there exists $\sigma \in S^{n-1}$ such that $H(\sigma)\neq 0$. Then $H \not \perp F_{k}(\sigma, \cdot)$ which implies that 
$$\operatorname{span}(\{F_k(\sigma,\cdot)\}_{\sigma\in S^{n}})= \mathcal H_k(S^{n}) \subset \mathscr{L}^2(S^{n}).$$

\section{Analysis} \label{sec:theory}

Given an activation function $\theta: \mathbb R \to \mathbb R$, then since $\{\varphi_{n,k}\}_{k=0}^\infty$ form an orthogonal basis of polynomials over $[-1,1]$ with respect to some measure, we can decompose $\theta$ as 
$$\theta(t)=\sum_{k=0}^\infty a_k \varphi_{n,k}(t),$$
for some $a_0, \ldots, a_k, \ldots \in \mathbb R$. Then 
\begin{equation}  \label{eq.repr}
(\theta \circ f_x)(y) = \theta(x\cdot y) = \sum_{k=0}^\infty a_k \varphi_{n,k}(x\cdot y) = \sum_{k=0}^\infty a_k F_k(x,y).
\end{equation}
In other words one layer of the SUNLayer neural network model~\eqref{eq.layer} can be expressed as $$L_{n}(x)=\theta\circ f_x = \sum_{k}^\infty a_k F_k(x,\cdot).$$

%
%
%

Note that if $\theta$ is a polynomial of degree $K$, then $L_{n}(S^{n}) =\{ \theta\circ f_{x} : x \in S^{n}\} \subset \oplus_{k=1}^K \mathcal H_K(S^{n})$ which is finite dimensional. Reciprocally, finite dimensional subspaces of $\mathscr L^2(S^{n})$ are included in $\oplus_{k=1}^K \mathcal H_K(S^{n})$ for some finite $K$. This observation, combined with the remark from Section~\ref{sec:finite_dim} suggest that polynomial activation functions are a useful model for studying the composition of multiple layers.

Lemma \ref{lem.max} shows an alternative expression for the least squares problem \eqref{lsp}.

\begin{lemma} For all $y\in \mathscr{L}^2(S^{n})$ we have 
$$\arg \min_{x\in S^{n}} \|\theta \circ f_{x} - y\|^2_{\mathscr{L}^2(S^{n})} = \arg \max_{x \in S^{n}}\langle \theta \circ f_{x}, y\rangle_{\mathscr{L}^2(S^{n})}.$$
\label{lem.max}
\end{lemma}

\begin{proof}
Note that for all rotations $Q\in O(n)$ we have 
$$\theta \circ f_{Qx}(z)= \theta(x^\top Q^\top z) = (\theta \circ f_x)(Q^\top z),$$
and so
\begin{eqnarray*} 
\|\theta\circ f_{Qx}\|^2_{\mathscr{L}^2(S^{n})} &=& \int_{z\in S^{n}} |(\theta\circ f_{Qx})(z)|^2dz = \int_{z\in S^{n}} |(\theta\circ f_{x})(Q^\top z)|^2 dz 
\\
&=& \int_{z\in S^{n}}|\theta \circ f_{x}(z)|^2 dz = \|\theta \circ f_{x}\|^2.
 \end{eqnarray*}
Therefore $\|\theta\circ f_x\|$ is constant for all $x\in S^{n}$, which implies the lemma since 
$$\|\theta\circ f_x - y \|^2 = \|\theta\circ f_x\|^2 + \|y^2\| - 2\langle \theta\circ f_x , y \rangle= \text{constant} - 2 \langle \theta\circ f_x , y \rangle.$$ \end{proof}

Given $y=\theta\circ f_{x^\sharp}$, according to Lemma \ref{lem.max} and equation \eqref{eq.repr}  we need to find $x\in S^{n}$ that maximizes
\begin{eqnarray} 
 \langle \theta\circ f_x , \theta\circ f_{x^\sharp} \rangle 
&=&  \sum_{k=0}^\infty \langle a_k F_k(x, \cdot), a_k F_k(x^\sharp, \cdot) \rangle \nonumber
=  \sum_{k=0}^\infty a_k^2 F_k(x, x^\sharp) \nonumber
\\
&=&  \sum_{k=0}^\infty a_k^2 \varphi_{n,k}(x\cdot x^\sharp) =:g_\theta(x\cdot x^\sharp). \label{eq.gtheta}
\end{eqnarray}
Note that the second equality is a consequence of the reproducing property \eqref{eq.reproducing}. The function $g_\theta$ will be particularly useful in our analysis.  
\begin{definition}
Let $\theta:\mathbb R \to \mathbb R$ be  an activation function, with Gegenbauer decomposition
$\displaystyle\theta(t)=\sum_{k=0}^\infty a_k\varphi_{n,k}(t)$. Then we define $\displaystyle g_\theta:\mathbb R \to \mathbb R$ as
$g_\theta(t)=\sum_{k=0}^\infty a_k^2 \varphi_{n,k}(t).$ 
\end{definition}

\begin{lemma}
\label{lemma.abs.convergence}
If $\theta(t):[-1,1]\to \mathbb R$ is $C^2$ and  $\theta(t) = \lim_{K\to \infty} \sum_{k=0}^K a_k \varphi_{n,k}(t)$ (convergence in $\mathscr{L}^2(\mu_n)$) then the functions  $g_\theta(t) = \displaystyle\lim_{K\to \infty} \sum_{k=0}^K a_k^2 \varphi_{n,k}(t)$ and $h_\theta=\displaystyle\lim_{K\to \infty} \sum_{k=0}^K a_k^2 \varphi_{n,k}'(t)$ are well-defined (and the convergence is also point-wise and absolute). Furthermore, if $\theta$ is $C^4$ we also have that $g_{\theta}'(t) = h_\theta(t)$ for all $t\in[-1,1]$.
\end{lemma}

\begin{proof} See Appendix \ref{app.convergence}.
\end{proof}

\begin{lemma} If $\theta(t)=\sum_{k=0}^\infty a_k\varphi_{n,k}(t)$ then 
$$c_{n,\theta}^2=\|\theta\circ f_x\|^2_{\mathscr{L}^2(S^{n})} = \operatorname{vol}(S^{n-1}) \|\theta\|^2_{\mathscr{L}^2(\mu_n)} =   g_\theta(1) $$ \label{lemma.theta}
\end{lemma}
\begin{proof}
Consider
$\displaystyle
\|\theta\circ f_x\|^2 = \int_{\tau\in S^n}\theta(x\cdot \tau)^2 d\tau.
$
Due to the rotational invariance observed in the proof of Lemma~\ref{lem.max} one take $x=(1,0,\ldots, 0)$, obtaining 
\begin{align*}
\|\theta\circ f_x\|^2 &= \int_{0}^\pi \theta(\cos(s))^2\sin(s)^{n-1} \operatorname{vol}(S^{n-1}) ds = \operatorname{vol}(S^{n-1}) \int_{-1}^1 \theta(t)^2 (1-t^2)^{(n-2)/2} dt \\ &= \operatorname{vol}(S^{n-1}) \|\theta\|^2_{\mathscr{L}^2(\mu_n)} \\
&= \operatorname{vol}(S^{n-1}) \int_{-1}^1 \left( \sum_{k=0}^\infty a_k \varphi_{n,k}(t) \right)^2 (1-t^2)^{(n-2)/2} dt \\
&=  \operatorname{vol}(S^{n-1})  \sum_{k=0}^\infty a_k^2  \int_{-1}^1\varphi_{n,k}(t)^2 (1-t^2)^{(n-2)/2} dt
 =  \sum_{k=0}^\infty \frac{ a_k^2 \alpha_{n,k} }{\operatorname{vol}(S^{n})} =  g_\theta(1) .
\end{align*}
The last line is due to Fubini-Tonelli and orthogonality of the Gegenbauer polynomials. The following equality is due to~\eqref{bound.norm} and the last equality is due to~\eqref{bound.inf}.
\end{proof}

\subsection{Noiseless case}
The following Theorem provides a sufficient condition that makes recovery possible in the noiseless case.

\begin{theorem}
Suppose $g_\theta'(t)>0$
for all $t\in[-1,1]$. Then for each $x^\sharp\in S^{n}$, the only
critical points of $$x\mapsto\|\theta\circ f_x-\theta\circ
f_{x^\sharp}\|^2$$ are $\pm x^\sharp$, with $x^\sharp$ being the unique
local minimizer.
\end{theorem}
\begin{proof}
Lemma \ref{lem.max} and equation \eqref{eq.gtheta} imply that critical points of $x\mapsto\|\theta\circ f_x-\theta\circ
f_{x^\sharp}\|^2$ coincide with critical points of $x\mapsto g_{\theta}(x\cdot{x^\sharp})$. In fact, local minima of the former correspond with local maxima of the latter. 
Using Lagrange multipliers we have $\mathcal L_{n}(x, \lambda)= g_\theta(x\cdot x^\sharp) + \lambda(\|x\|^2-1)$ which gives optimality conditions
\begin{align*}
\left\{ \begin{matrix*}[l]0=\nabla_x\mathcal L_{n} = g_{\theta}'(x\cdot x^\sharp) x^\sharp + 2\lambda x \\
0= \frac{\partial}{\partial \lambda} \mathcal L_{n} = \|x\|^2-1 \end{matrix*}\right.
\end{align*}
If $g_\theta'(t)\neq 0$ for all $t\in[-1,1]$, then $\lambda\neq 0$ which implies $x=\frac{g_\theta'(x\cdot x^\sharp)}{-2\lambda} x^\sharp=\pm x^\sharp$. Since $g_\theta(1)>g_\theta(-1)$ then $x=x^\sharp$. 
\end{proof}

\subsection{Denoising}
The following Theorem is the main result of this paper.
\begin{theorem} \label{thm.denoising}
Let $\displaystyle\theta(t)=\sum_{k=0}^K a_k \varphi_{n,k}(t)$ and $y= \theta \circ f_{x^\sharp} + \eta$. We decompose $\eta$ as follows:
$$\eta = \sum_{k=0}^\infty \sum_{i=1}^{d_k} e_{k,i} F_k(\sigma_{k,i},\cdot)=: \sum_{k=0}^\infty \eta_k \text{ with } \eta_k\in \mathcal H_k(S^{n}).$$
Denoting $\displaystyle \epsilon := \sup_{x\in S^n}\left\| \sum_{k=0}^K a_k \sum_{i=1}^{d_k}e_{k,i}\varphi'_{n,k}(x\cdot \sigma_{k,i}) \sigma_{k,i} \right \|$ and $ T := \inf_{t\in[-1,1]} |g_\theta'(t)|$, suppose the noise $\eta$ is small enough that $\epsilon<T$. Then
\begin{enumerate}
\item[\textbf{(a)}] Every critical point $\hat x$ of  
$ x \mapsto \|\theta \circ f_{x} - y\|^2_{\mathscr{L}^2(S^{n})}$
satisfies that $| \hat x\cdot x^\sharp | > 1 -\tfrac{2\epsilon}{T+\epsilon}$
\item[\textbf{(b)}] Define $M_k:=\max_{t\in [-1,1]} |\varphi'_{n,k}(t)|$, then $\displaystyle \epsilon \leq \sum_{k=1}^K M_k|a_k|\|\eta_k\|$.
\end{enumerate}
\end{theorem}
\begin{proof}\textbf{of Theorem \ref{thm.denoising} \textit{(a)}} 
According to Lemma \ref{lem.max} we need to solve 
$$\max_{x\in S^{n}} \langle \theta\circ f_x , \theta\circ f_{x^\sharp} + \eta \rangle = \max_{x\in S^{n}} g_\theta(x\cdot x^\sharp) + \langle \theta \circ f_x, \eta \rangle.$$
The reproducing property implies
\begin{eqnarray*}\langle \theta\circ f_x, \eta \rangle 
&=& \sum_{k=0}^K \langle a_k F_k(x,\cdot), \sum_{i=1}^{d_k} e_{k,i} F_k(\sigma_{k,i}, \cdot) \rangle 
\\
&=& \sum_{k=0}^K a_k \sum_{i=1}^{d_k} e_{k,i} \varphi_{n,k}(x\cdot \sigma_{k,i}).
\end{eqnarray*}
Therefore the denoising objective is 
\begin{equation}
\max_{x\in S^{n}} g_\theta(x\cdot x^\sharp) + \sum_{k=0}^K a_k \sum_{i=1}^{d_k} e_{k,i}\varphi(x\cdot \sigma_{k,i}) \label{objective.noise}
\end{equation}
For $x$ critical point of \eqref{objective.noise} Lagrange multipliers give us
$$\mathcal L_{n}(x, \lambda)= g_\theta(x\cdot x^\sharp) + \sum_{k=0}^K a_k \sum_{i=1}^{d_k} e_{k,i} \varphi(x\cdot \sigma_{k,i}) + \lambda( \|x\|^2 -1)$$
and $\frac{\partial}{\partial x}\mathcal L_{n} =0$ implies
$$0=\underset{A}{\underbrace{g_\theta'(x\cdot x^\sharp)x^\sharp}} + \underset{B}{\underbrace{\sum_{k=0}^K a_k \sum_{i=1}^{d_k}e_{k,i}\varphi'_{n,k}(x\cdot \sigma_{k,i}) \sigma_{k,i}}} +  2 \lambda x$$

By hypothesis, we have $\|B\| \leq \epsilon < T \leq \|A\|$, and so
\[
\|2\lambda x\| = \|A+B\|\geq \|A\|-\|B\| > 0,
\]
which implies $\lambda\neq 0$.
Therefore
$$x=\frac{-1}{2\lambda}(g_\theta'(x\cdot x^\sharp )x^\sharp +B)$$
and 
$$2|\lambda|= \|g_\theta'(x\cdot x^\sharp)x^\sharp + B \| \leq |g_\theta'(x\cdot x^\sharp)| +\epsilon,$$
and so
$$|x\cdot x^\sharp| = \frac{1}{2|\lambda|}|g_\theta'(x\cdot x^\sharp) + B\cdot x^\sharp| \geq \frac{|g_\theta'(x\cdot x^\sharp)|}{2|\lambda|}\geq \frac{|g_\theta'(x\cdot x^\sharp)| - \epsilon}{|g_\theta'(x\cdot x^\sharp)| + \epsilon} \geq 1 - \frac{2\epsilon}{T+\epsilon}.$$
\end{proof}

The key parameter $\epsilon$ in Theorem \ref{thm.denoising} depends on both the noise $\eta$ and the activation function $\theta$. In order to understand the behavior of $\epsilon$ in terms of the noise $\eta$, and prove Theorem \ref{thm.denoising} \textit{(b)}, we choose $\{\sigma_{k,i}\}_i$ ($i=1,\ldots N)$ so that  $\{ F_k(\sigma_{k,i},\cdot)\}_i$ forms a tight frame. To this end it suffices for $\{\sigma_{k,i}\}_i$ to form a spherical $t$-design for $t=2k$.

\begin{definition}[Spherical $t$-design]
A spherical $t$-design is a sequence of $N_t$ points $\{x_1, \ldots, x_{N_t}\}\subset S^{n}$ such that for every polynomial $p$ of degree at most $2t$ we have
$$\frac{1}{N_t}\sum_{i=1}^{N_t} p(x_i) = \int_{S^{n}}p(x)dx.$$
\end{definition}

\begin{definition}[Tight frame]
Let $(V,\langle \cdot, \cdot \rangle)$ be vector space with an inner product. A tight frame is a sequence $\{v_k\}_{k\in I \subseteq \mathbb N} \subset V$ such that there exists a constant $c$ so that  for all $v\in V$
$$\sum_{k\in I} |\langle v, v_k\rangle|^2=c\|v\|^2.$$ 
\end{definition}

\begin{lemma}
If $\{\sigma_{k,i}\}_{i=1}^N$ form a spherical $t$-design with $t=2k$ then $\{ F_k(\sigma_{k,i},\cdot)\}_i$ is a tight frame for $\mathcal H_k(S^n)$ with constant $c=N_k$.
\end{lemma}

\begin{proof} Let $\{Y_j\}$ be an orthonormal basis for $\mathcal H_k(S^n)$. Consider $\delta_{a,b}=1$ if $a=b$ and 0 otherwise. It suffices to show 
\begin{eqnarray*}
N_k\delta_{j,j'} &=& \sum_{i=1}^{N_k}\langle F_k(\sigma_{k,i}, \cdot), Y_j \rangle \overline{\langle F_k(\sigma_{k,i}, \cdot), Y_{j'}\rangle} \\
&=& \sum_{i=1}^{N_k}\left(\int_{\tau\in S^n}  F_k(\sigma_{k,i}, \tau) \overline{Y_j(\tau)}\right) \overline{\left(\int_{\tau\in S^n}  F_k(\sigma_{k,i}, \tau) \overline{Y_{j'}(\tau)}\right) } \\
&=& \sum_{i=1}^{N_k}\left(\sum_{j''}Y_{j''}(\sigma_{k,i})\int_{\tau\in S^n} Y_{j''}(\tau) \overline{Y_j(\tau)}\right) \overline{ \sum_{j'''}Y_{j'''}(\sigma_{k,i})\int_{\tau\in S^n} Y_{j'''}(\tau) \overline{Y_{j'}(\tau)} } 
\\
&=&\sum_{i=1}^{N_k}Y_j(\sigma_{k,i})\overline{Y_{j'}(\sigma_{k,i})} 
\end{eqnarray*}
Observe that if $Y_j, Y_{j'}\in \mathcal H_k(S^n)$ then $p(x)=Y_{j}(x)Y_{j'}(x)$ is a polynomial of degree $2k$. Then using the $t$-design property we get
\begin{eqnarray*}
\sum_{i=1}^{N_k}Y_j(\sigma_{k,i})\overline{Y_{j'}(\sigma_{k,i})} &=& N_k \frac{1}{N_k} \sum_{i=1}^{N_k}Y_j(\sigma_{k,i})\overline{Y_{j'}(\sigma_{k,i})} \\
&=& N_k \int_{\tau\in S^n}Y_{j}(\tau)\overline{Y_{j'}(\tau)} \\
&=& N_k\langle Y_j, Y_{j'} \rangle_{\mathcal H_k(S^n)}
\end{eqnarray*}
which proves the theorem.
\end{proof}

\begin{proof}[Proof of Theorem \ref{thm.denoising} \textit{(b)}]
We choose $\{\sigma_{k,i}\}_{i=1}^N$ so that $\{ F_k(\sigma_{k,i},\cdot)\}_i$ is a tight frame for $\mathcal H_k(S^n)$ with constant~$N_k$. We write $\eta = \sum_{k=0}^\infty \sum_{i=1}^{d_k} e_{k,i} F_k(\sigma_{k,i},\cdot)$. One can uniquely decompose
$\eta= \sum_{k=0}^{\infty} \eta_k$ with $\eta_k\in \mathcal H_k(S^n)$ and we have $\sum_{i}|e_{k,i}|^2 = \frac{1}{N_k} \|\eta_k\|^2$. In fact $e_{k,i}$ can be chosen so that $e_{k,i}=\frac{1}{N_k}\langle \eta, F_k(\sigma_{k,i},\cdot) \rangle$.
Following the notation in the proof of Theorem \ref{thm.denoising} \textit{(a)} we have:
$$B=\sum_{k=1}^K a_k \sum_{i=1}^{N_k} e_{k,i} \varphi'_{n,k}(x\cdot \sigma_{k,i})\sigma_{k,i}$$
and when $x\in S^n$ maximizes $\|B\|$, we have
$$\epsilon = \|B\|\leq \sum_{k=1}^K |a_k| \left \| \sum_{i=1}^{N_k}  e_{k,i} \varphi'_{n,k}(x\cdot \sigma_{k,i})\sigma_{k,i} \right \|$$ 
Let $G_{k,x}: \mathscr{L}^2(S^{n})\to S^{n}$ such that $G_{k,x}(\eta)= \sum_{i=1}^{N_k} \frac{1}{N_k} \langle e, F_k(\sigma_{k,i}, \cdot) \rangle \varphi_{n,k}'(x\cdot \sigma_{k,i})\sigma_{k,i}$. Let 
$$\|G_{k,x}\|_{2\to 2} = \sup_{\|\nu\|_{\mathscr{L}^2(S^{n})} =1 }\|G_{k,x}(\nu) \|_{S^{n}}$$
then for all $x\in S^{n}$ we have
$$\epsilon \leq \left( \sum_{k=0}^K |a_k| \|G_{k,x}\|_{2\to 2} \right)\|\eta\|.$$
Since $M_k=\max_{t\in [-1,1]} |\varphi'_{n,k}(t)|$  we bound
$$\max_{x}\|G_{k,x}\|_{2\to2} \leq \frac{M_k}{N_k}\sum_{i=1}^{N_k} \langle \eta, F(\sigma_{k,i}, \cdot)\rangle  \|\sigma_{k,i}\| = M_k \sum_{i=1}^{N_k} |e_{k,i}|,$$
obtaining the bound 
$$\epsilon \leq \sum_{k=1}^K M_k |a_k| \sum_{i=1}^{N_k}|e_{k,i}|.$$
Using Theorem \ref{thm.denoising} \textit{(a)} we conclude that denoising is possible provided that 
$$\sum_{k=1}^K M_k |a_k| \sum_{i=1}^{N_k}|e_{k,i}| \leq \inf_{t\in[-1,1]} \sum_{k=1}^K a_k^2 \varphi_{n,k}'(t).$$
Note that the left hand side depends on the activation function $\theta$ and the noise $\eta$ whereas the right hand side depends only on $\theta$. Using the frame properties and Cauchy-Schwarz inequality one can write
$$\epsilon \leq \sum_{k=1}^k M_k |a_k| \|\eta_k\|$$
and prove Theorem \ref{thm.denoising} \textit{(b)}. Note that this implies a sufficient condition for denosing
$$\sum_{k=1}^k M_k |a_k| \|\eta_k\| < \inf_{t\in[-1,1]} \sum_{k=1}^K a_k^2 \varphi_{n,k}'(t), $$
that does not depend on the frame choice.
\end{proof}
\section{Numerical experiments}

\begin{figure}
\hspace{-0.06\textwidth}
\includegraphics[width=0.33\textwidth]{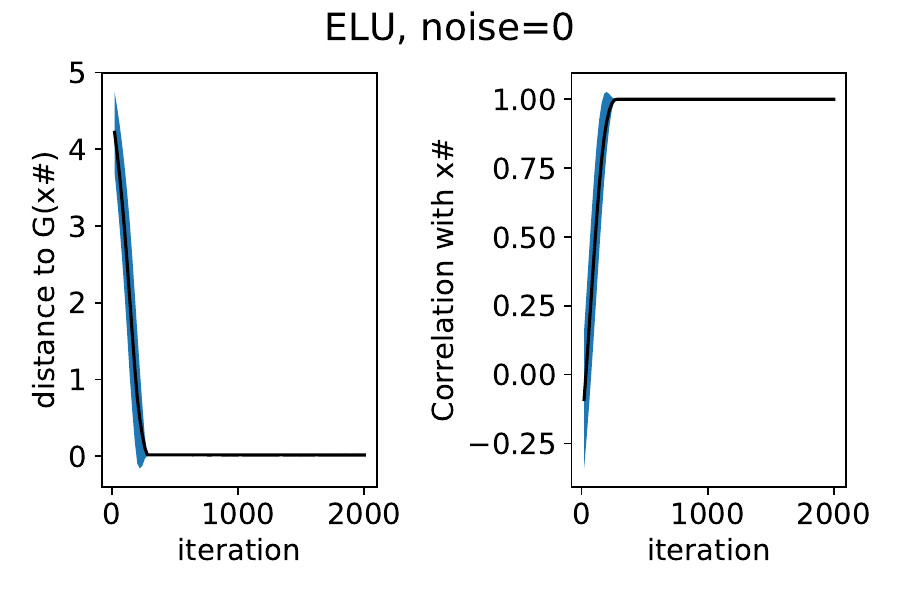}\hspace{-0.01\textwidth}
\includegraphics[width=0.33\textwidth]{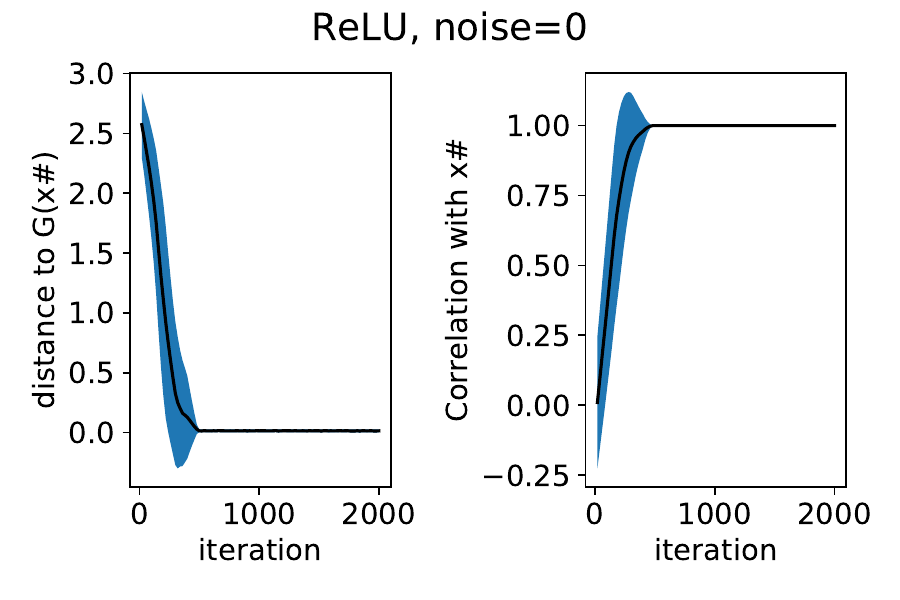} \hspace{-0.02\textwidth}
\includegraphics[width=0.33\textwidth]{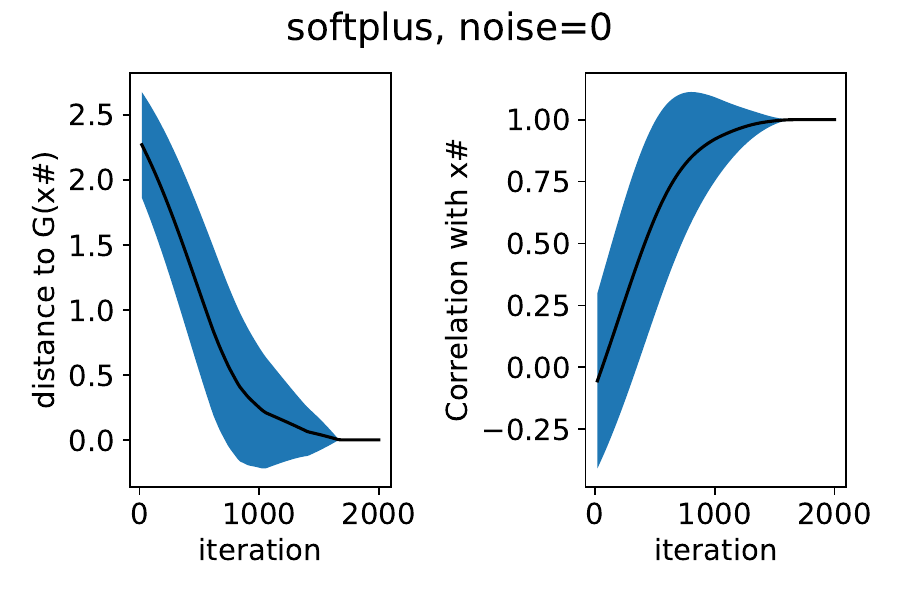}

\hspace{-0.06\textwidth}
\includegraphics[width=0.33\textwidth]{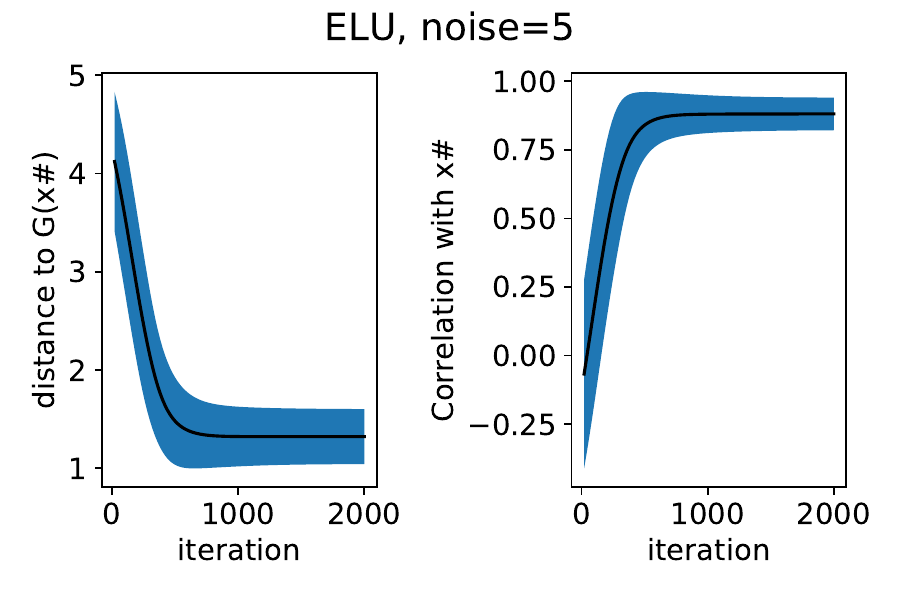}\hspace{-0.01\textwidth}
\includegraphics[width=0.33\textwidth]{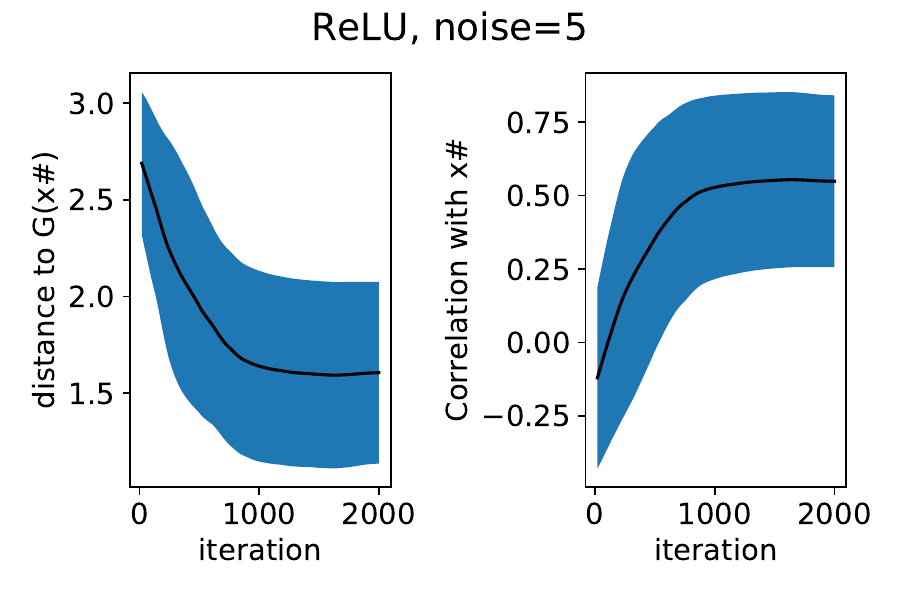} \hspace{-0.02\textwidth}
\includegraphics[width=0.33\textwidth]{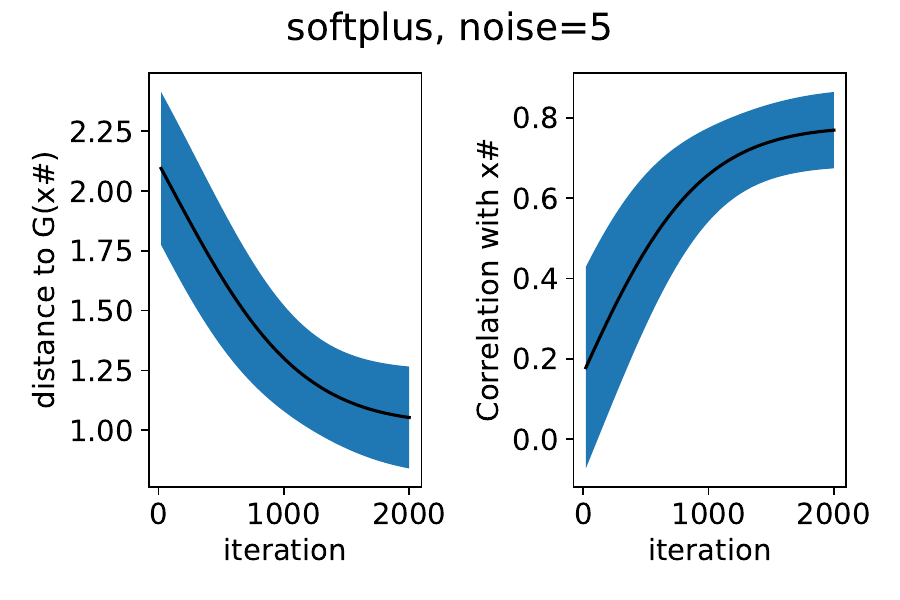}

\hspace{-0.06\textwidth}
\includegraphics[width=0.33\textwidth]{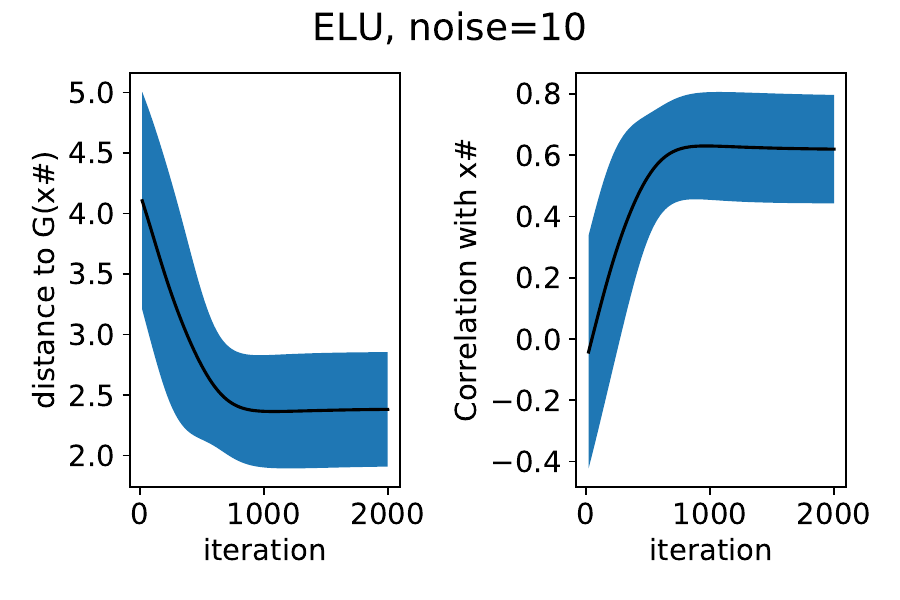}\hspace{-0.01\textwidth}
\includegraphics[width=0.33\textwidth]{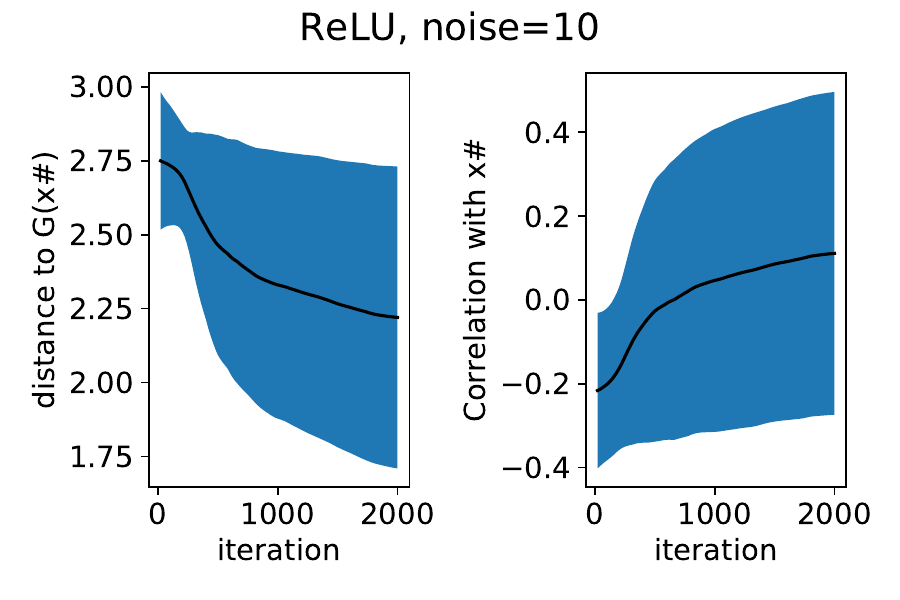} \hspace{-0.02\textwidth}
\includegraphics[width=0.33\textwidth]{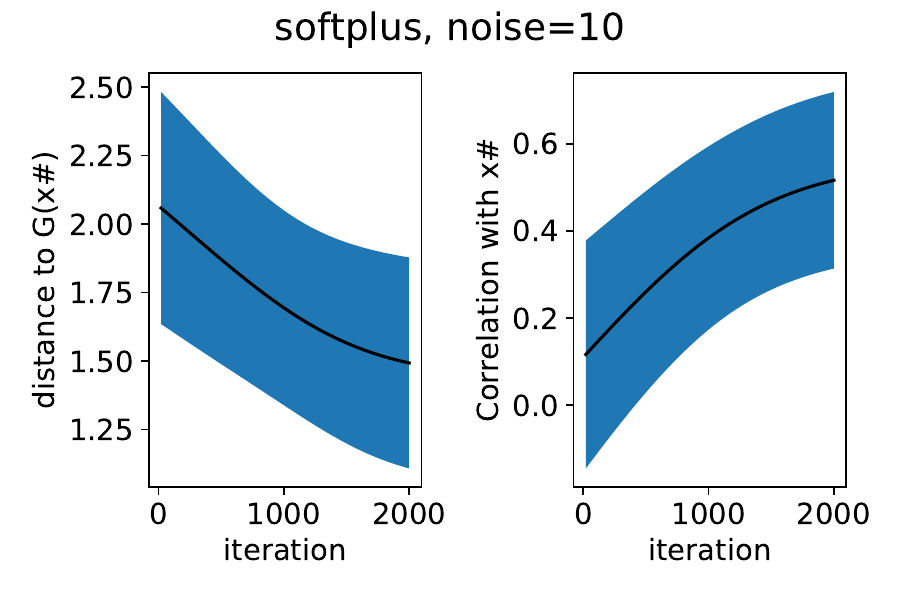}
\caption{Denoising performance of the SUNLayer for different activation functions}{\label{fig:synthetic} We consider ReLU, ELU and softplus. The denoising performance of ELU is superior to ReLU and softplus, which is consistent with the theory from Section~\ref{sec:theory} and the properties shown in Figure~\ref{fig.activation2}. Note that softplus satisfies that $g_\theta'(t)>0$ for all $t\in[-1,1]$ but the values of $g_\theta'(t)$ are close to zero. }
\end{figure}   

\subsection{Denoising a generative model for MNIST}
Figure~\ref{fig.denoising} shows denosing of simple images using a generative model in comparison with classical denoising algorithms. In order to produce the generative model we use a 4-layer convolutional neural network with ELU as the activation function. We train it in the entire MNIST training set using stochastic gradient descent.

\subsection{Gegenbauer approximations of common activation functions}
Figures \ref{fig.activation} and \ref{fig.activation2} consider the most common activation functions. The first column plots the activation function, the second column shows the Gegenbauer approximation truncating to $K=30$. The third column plots $g_\theta(t)$ for the approximation of $\theta$ from the second column and the fourth column plots $g_\theta'(t)$. The difference between Figures \ref{fig.activation} and \ref{fig.activation2} is the space considered (in Figure \ref{fig.activation} $n=2$ whereas in Figure \ref{fig.activation2} $n=10$). According to Theorem \ref{thm.denoising}, the best activation functions for denoising will be the ones where $g_\theta'(t)$ is bounded away from zero, in particular we observe that the performance for all nonlinearities seem to deteriorate by increasing $n$. We also observe that ELU~\cite{clevert2015fast}
theoretically has better denoising properties than  Softplus~\cite{nair2010rectified}, LeakyReLU~\cite{maas2013rectifier}, Swish~\cite{ramachandran2017swish} or ReLU. 

Table \ref{table} shows lower bounds for $\inf_{t\in[-1,1]} |g'_{\theta}(t)|$ for popular activation functions. The strategy to produce such lower bounds mainly uses the bound \eqref{c2.bound} and it is explained in Appendix \ref{app.bounds}. Note that we do not produce a provable bound for ELU or ReLU because they are not smooth enough.

\begin{table}
\centering
\begin{tabular}{ccccc}
\hline
\hline
Activation function & n & $T:=\inf_{t\in[-1,1]} |g'_{\theta}(t)|$ & $c_{\theta,n}$ & $T/c_{\theta,n}$ \\
\hline
\hline
$\operatorname{id}(x)=x$ & 2 & 4.189 & 4.189 &1.000
\\
& 10 &  1.884 & 1.884  & 1.000  
\\
\hline
$\operatorname{softplus}(x)=\log(1+e^x)$ & 2 & 0.998 & 7.83 & 0.127 
\\
 & 10 & 0.462 & 10.759 & 0.043
 \\
 \hline
$\displaystyle\operatorname{tanh}(x)=\frac{e^x - e^{-x}}{e^x + e^{-x}}$ & 2 & 2.959 & 2.996 & 0.988
\\
 & 10 & 1.635 & 1.639 & 0.997
 \\
 \hline
 $\operatorname{sigmoid}(x)=\displaystyle\frac{1}{1 + e^{-x}}$ & 2 & 0.238 & 3.380 & 0.071
\\
 & 10 & 0.113 & 5.295 & 0.021
 \\
 \hline 
 $\operatorname{swish}(x)=\displaystyle\frac{x}{1 + e^{-x}}$ & 2 & 0.864 & 1.187 & 0.728
\\
 & 10 & 0.437 & 0.497 & 0.880
 \\
 \hline 
 \hline 
\end{tabular}
\caption{\label{table} Lower bounds for $\inf_{t\in[-1,1]} |g'_{\theta}(t)|$. We use a degree 10 polynomial approximation. The approximation errors in the entries of this table are smaller than $10^{-3}$.}
\end{table}

\begin{figure}[h]
\centering
\includegraphics[width=\textwidth]{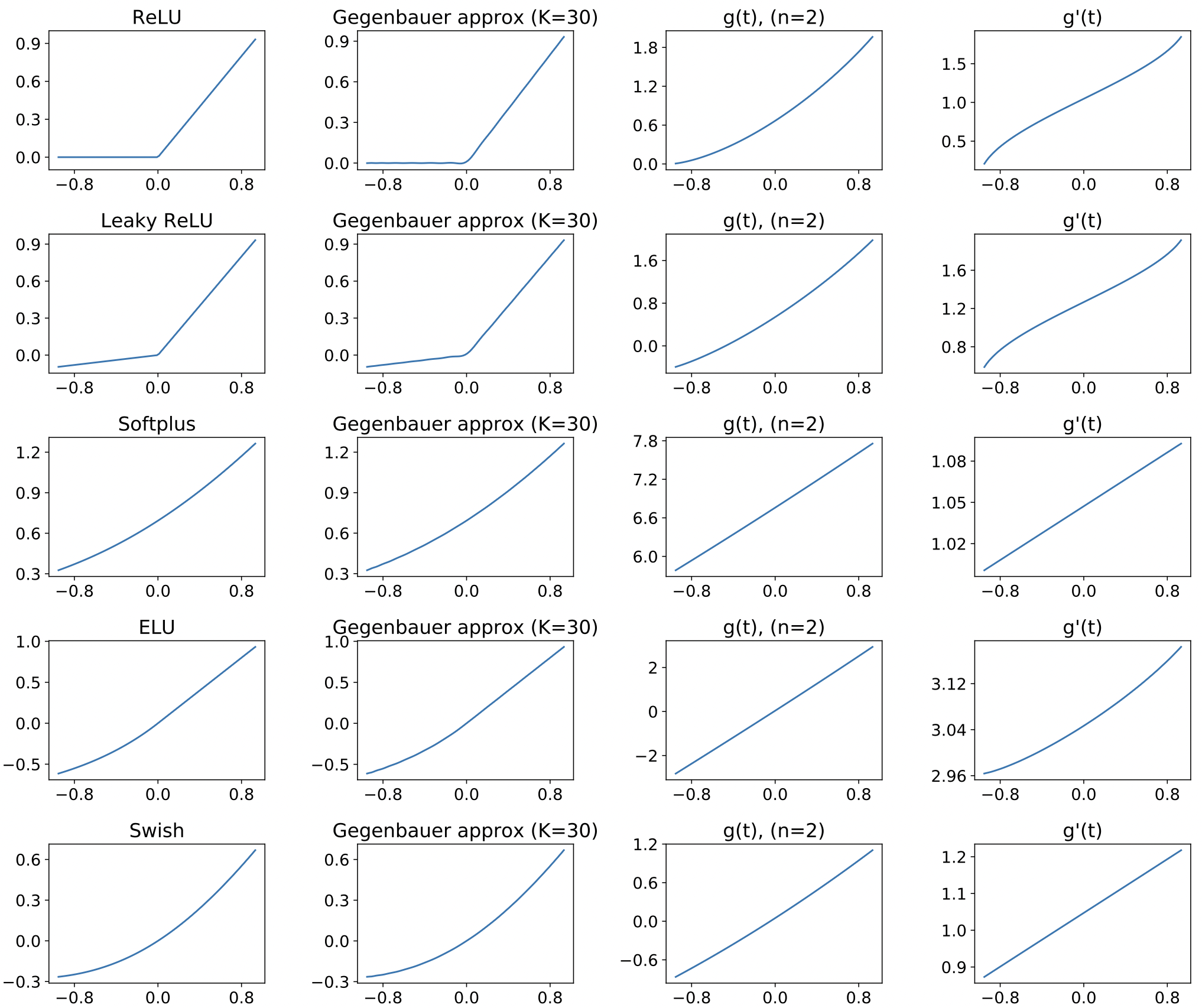}
\caption{\label{fig.activation} Activation functions and their Gegenbauer approximations for $K=30$ and $n=2$.}
\end{figure}

\begin{figure}[h]
\centering
\includegraphics[width=\textwidth]{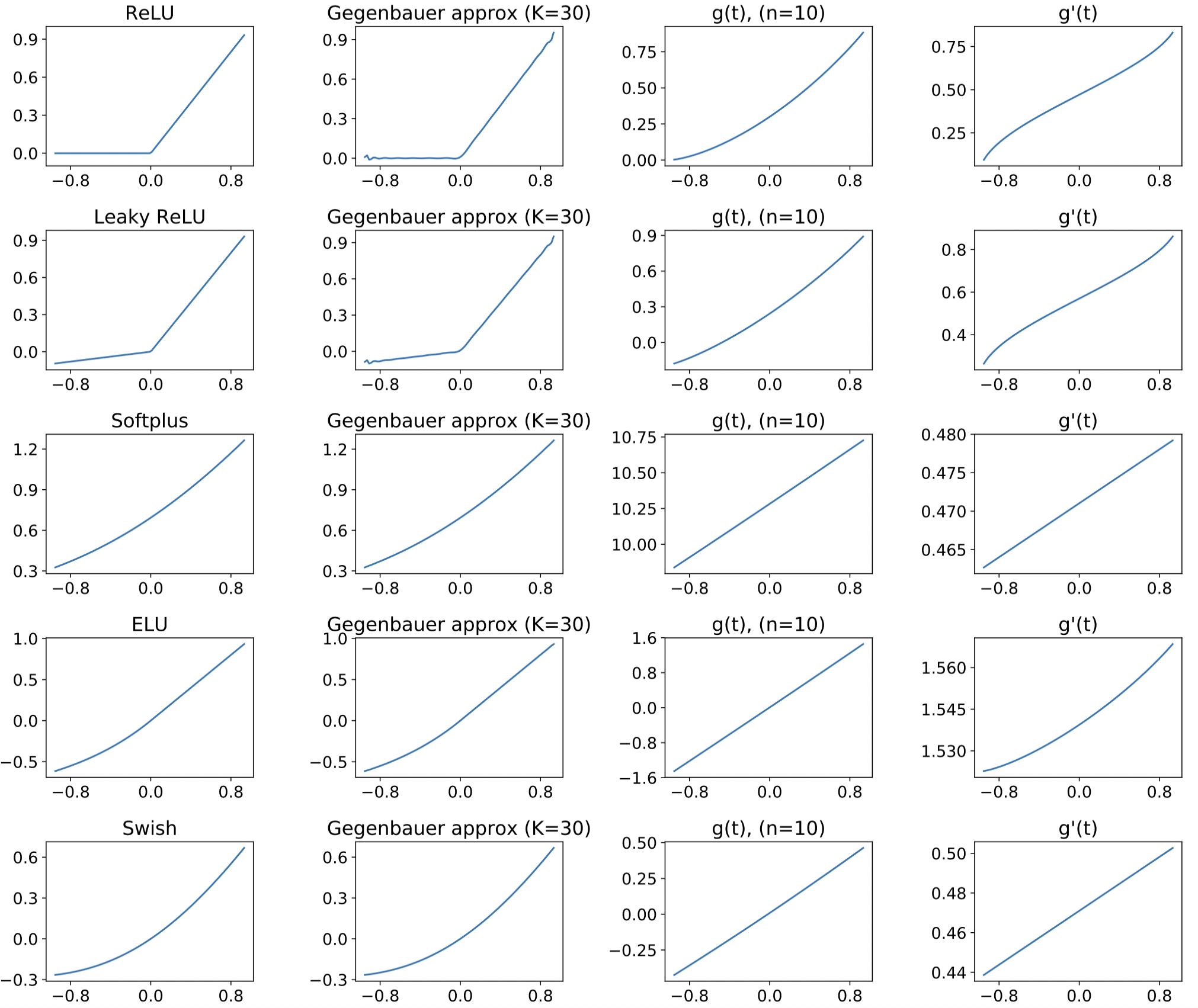}
\caption{\label{fig.activation2} Activation functions and their Gegenbauer approximations for $K=30$ and $n=10$.}
\end{figure}

\subsection{Denoising in a synthetic framework} \label{sec:synthetic}
In Figure~\ref{fig:synthetic} we perform a numerical experiment to illustrate the theory developed in Section \ref{sec:theory}. We consider a random instance of one layer of the SUNLayer model. Here $G(x)=\theta(Bx)$ where $x\in S^{n}$ for $n=9$ and $B\in \mathbb R^{100\times 10}$ is a fixed random Gaussian matrix with normalized rows. 
We perform 10 independent experiments where we draw random $x^\sharp \in S^9$ and we let $y=G(x^\sharp) + \eta$ where $\eta$ is Gaussian in $\mathbb R^{100}$. For each $y$ we use stochastic gradient descent to find $\hat x$, a local minimizer of $ \| G(x) - y \|$. 

We report $\| G(\hat x) - G(x^\sharp)\|$ and $\langle \hat x, x^ \sharp \rangle$ for different noise levels and activation functions. The solid curve corresponds to the mean over the 10 experiments, whereas the shaded area shows the standard deviation. The activation functions we consider are ReLU, softplus and ELU. The denoising performance of ELU is empirically superior to ReLU and softplus. This observation is consistent with the theory from Section~\ref{sec:theory} and the properties shown in Figure~\ref{fig.activation2}. Note that softplus satisfies that $g_\theta'(t)>0$ for all $t\in[-1,1]$ but the values of $g_\theta'(t)$ are close to zero.

\subsection{Denoising for tomography-type data }

This section presents numerical experiments on generative networks that resemble the SUNLayer structure for image denoising, evaluated on Shepp–Logan image classes. In particular, we investigate the stability of the denoising performance with respect to the choice of activation function, focusing on ReLU and ELU. The results indicate that ELU networks yield more stable reconstructions than ReLU (i.e. the ELU denoising networks are better conditioned than the ReLU ones). These findings are consistent with the properties established in Theorem~\ref{thm.denoising} and demonstrate that the proposed criteria is potentially useful for analyzing local denoising methods.

We begin by outlining the experimental setup, which includes the network framework, evaluation methodology, and testing data. We then present the numerical results and highlight our discoveries.

\paragraph{Generative network architecture.}

We consider a denoising autoencoder with encoder $E_\theta$ and decoder $D_\theta$:
\begin{equation}
    z = E_\theta(y), \quad \hat{x} = D_\theta(z),
\end{equation}
where $y \in \mathbb{R}^{n \times n}$ is a noisy image, $z \in \mathbb{R}^d$ is the latent code, and $\hat{x}$ is the reconstruction.
The encoder consists of three convolutional layers (linear) with nonlinear activations, followed by a fully-connected layer projecting to the latent space. The decoder mirrors this structure using transposed convolutions to reconstruct the image.

We enforce a soft normalization on the latent code:
\[
\tilde{z} = \frac{z}{1 + \|z\|_2},
\]
which encourages latent vectors to lie near the unit sphere, to resemble the SUNLayer model. The overall training objective is to minimize
\[
\mathbb{E} \left[ \|\hat{x} - x\|_2^2 \right].
\]

\paragraph{Activation functions.}

We investigate the effect of different nonlinear activation functions $\theta(\cdot)$ on local stability of denoising recovery:

\begin{itemize}
    \item \textbf{ReLU:} $\theta(x) = \max(0, x)$.
    \item \textbf{ELU:} $\theta(x) = x$ if $x > 0$, else $e^x - 1$.
\end{itemize}
The choice of activation directly affects the local properties of the trained denoising network, and thus can influence the stability of the reconstruction. However, in this experiment, we do not compare the overall denoising performance or accuracy of the two activation functions. Instead, we ensure that their denoising results are at comparable levels and focus on analyzing the detailed local properties of the optimized outputs.

\paragraph{Evaluation.}
At a high level, Theorem~\ref{thm.denoising} also shows that \textbf{local recovery stability} of a generative network $G = \theta \circ f$ depends on two key factors from the activation functions that we will call \emph{sensitivity} $T$ and \emph{amplification} $M_k$.  
    \[
    T = \inf_{t \in [-1,1]} |g_\theta'(t)|,\quad 
    M_k = \max_{t \in [-1,1]} |\varphi'_{n,k}(t)|,~\forall k \leq K.
    \]
Together, these quantities govern the local stability of the generative network $G$.

To connect the experiments to the framework of Theorem~\ref{thm.denoising}, we focus on the decoder $D_\theta$, which maps latent codes to reconstructed images. Importantly, the decoder 
relates to the generative network $G$ in our theory via $D_\theta\approx G^{-1}$. See Figure~\ref{fig: flowchart} for architecture illustration. In particular, we examine:

\begin{itemize}
    \item The minimal and maximal singular values of the decoder Jacobian $J_D(z) = \partial D_\theta(z)/\partial z$, i.e. $\sigma_{\min}(J_D(z))$ and $\sigma_{\max}(J_D(z))$, which quantify local conditioning of the map. A larger minimal singular value indicates greater local stability, while a larger maximal singular value reflects worse sensitivity to perturbations.
    \item The local Lipschitz constant of $D_\theta$ in regions of the latent manifold, providing a measure of how perturbations in $z$ propagate to the output. Smaller Lipschitz constants correspond to a more stable local geometry.
  
\end{itemize}
   We note that we are comparing the stability of the trained denoisers using the methodology from \cite{vincent2008extracting} for different activation functions. The main Theorem~\ref{thm.denoising} does not directly apply to this numerical setting, but the the experiments show that the behavior of the generative model is consistent to what the theory predicts in the idealized setting. 
\begin{center}
\begin{tikzpicture}[node distance=0.3cm, auto, font=\scriptsize]

\node [draw, rectangle, minimum width=2.5cm, minimum height=1cm] (y) {Noisy image $y$};
\node [draw, rectangle, right=of y, minimum width=2.5cm, minimum height=1cm] (encoder) {Encoder $E_\theta$};
\node [draw, rectangle, right=of encoder, minimum width=2.5cm, minimum height=1cm] (latent) {Latent code $z$};
\node [draw, rectangle, right=of latent, minimum width=2.5cm, minimum height=1cm] (decoder) {Decoder $D_\theta$};
\node [draw, rectangle, below=of decoder, minimum width=3cm, minimum height=1cm] (tests) {Stability evaluation:  $\sigma_\text{min}(J), \sigma_\text{max}(J), \text{Lip.~const.}$};
\node [draw, rectangle, above=of decoder, minimum width=1cm, minimum height=1cm] (G) {$G$};
\node [draw, rectangle, right=of decoder, minimum width=3cm, minimum height=1cm] (reconstruct) {Reconstruction $\hat{x}$};

\draw[->, thick] (y) -- (encoder) node[midway, above] {};
\draw[->, thick] (encoder) -- (latent) node[midway, above] {};
\draw[->, thick] (latent) -- (decoder) node[midway, above] {};
\draw[->, thick] (decoder) -- (reconstruct) node[midway, above] {};
\draw[->, thick, dashed] (latent) -- (tests) node[midway, left] {};
\draw[->, thick] (decoder.south) -- (tests.north) node[midway, right] {};
\draw[->, thick, bend right=30] (reconstruct) to (G) node[midway, left] {};
\draw[->, thick, bend right=30] (G) to (latent) node[midway, left] {};
\end{tikzpicture}
\captionof{figure}{The dashed arrow indicates that stability tests are performed on the decoder outputs as a function of latent perturbations.}
\label{fig: flowchart}
\end{center}

\paragraph{Testing data.}

All experiments are conducted on the Shepp-Logan phantom class \cite{shepp1974}, a standard benchmark for head section imaging. The original phantom consists of a superposition of $10$ ellipses, each parameterized by rotation angle, semi-major, semi-minor axes, position, and intensity.

To generate a dataset of similar images, we introduce random tweaks to each ellipse independently. Specifically, for each ellipse we apply random rotation, with angle sampled uniformly from $[-30^\circ, 30^\circ]$; and independent scaling of the semi-major and semi-minor axes, sampled from $\mathrm{Uniform}[0.8, 1.2]$.

We consider two image classes:
\begin{itemize}
    \item \textbf{Class 1}: random rotation only;
    \item \textbf{Class 2}: random rotation combined with random axis scaling.
\end{itemize}
Clearly, $\text{Class 1} \subseteq \text{Class 2}$, hence denoising Class 2 images is a more challenging task than on Class 1. 
All images are generated on a $128 \times 128$ pixel on the grid $[-1, 1]^2$. After random phantom generation, each image is normalized to have zero mean and unit variance.

Additive Gaussian noise is then applied pixel-wise:
\[
y = x + \eta, \qquad \eta \sim \mathcal{N}(0, \epsilon^2 I),
\]
where we consider two noise levels:
\[
\begin{array}{l}
\text{medium:} ~ \epsilon = 1, ~~
\text{high:}  ~ \epsilon = 5.
\end{array}
\]
A demonstration of the generated Shepp-Logan image classes is in Figure~\ref{fig: shepp-logan}.

\begin{figure}[!htb] 
\hspace{-30mm} \includegraphics[width=1.4\linewidth]{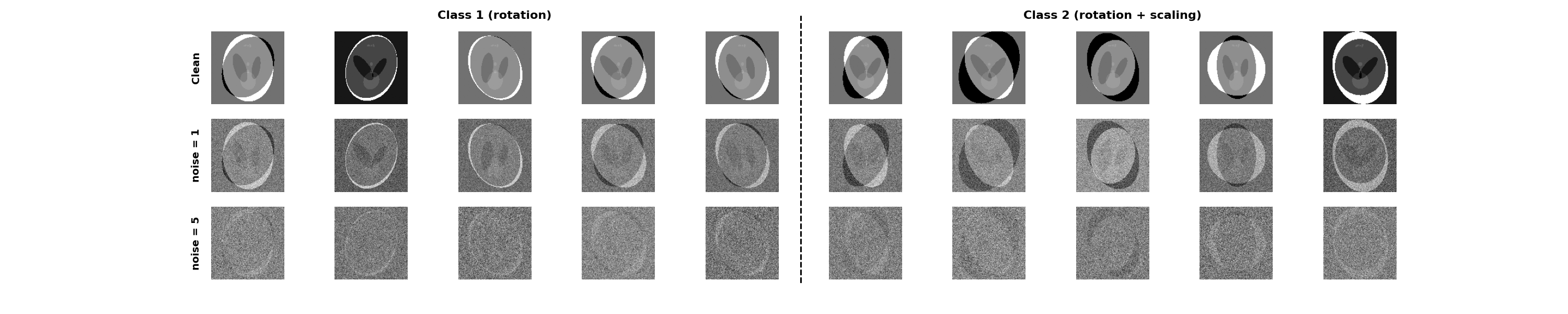} 
\caption{Shepp–Logan image classes (top row) and their noisy versions at different noise levels (middle and bottom). } \label{fig: shepp-logan} \end{figure}

\paragraph{Numerical results.} 

On each class of images and associated noise levels, we use 1000 training samples to build a sufficiently good denoising encoder-decoder. We test on 100 samples from the same class of noisy images. Below we present the denoising visualization in Figures~\ref{fig: easy-1},~  \ref{fig: easy-5},~ \ref{fig: hard-5}, where we randomly sample 2 testing images for demonstration. The stability evaluation on the two activation functions: ReLU and ELU, are summarized below in Tables \ref{table: easy-1},~\ref{table: easy-5},~\ref{table: hard-5}, where we highlight the best result per testing case in \textbf{bold} font. 

In summary, ELU outperforms ReLU in almost every metrics in all the testing cases regarding the reconstruction stability, achieving smaller magnitude control in the local landscape by observing its Lipschitz constant and $\sigma_{\max}(J)$, and (sometimes) more robust $\sigma_{\min}(J)$ as well, see Table~\ref{table: easy-5}. Autoencoders employing both activation functions exhibit meaningful denoising performance, without visually distinguishable results.

\begin{center}
\begin{tabular}{lccc}
\toprule\textbf{Activation} & \textbf{$\sigma_{\min}(J) \uparrow$} & \textbf{$\sigma_{\max}(J) \downarrow$} & $\text{Lip. const.} \downarrow$ \\
\midrule\textbf{ReLU} & \bf 1.60 & 15.314 & 141.819 \\
\textbf{ELU} & 1.370 & \bf 10.029 & \bf 93.541 \\
\bottomrule
\end{tabular}
\captionof{table}{Class 1 with noise level $\epsilon = 1:$ stability evaluation.}
\label{table: easy-1} 
\end{center}

\begin{figure}[!htb]
    \centering
\includegraphics[width=0.9\linewidth]{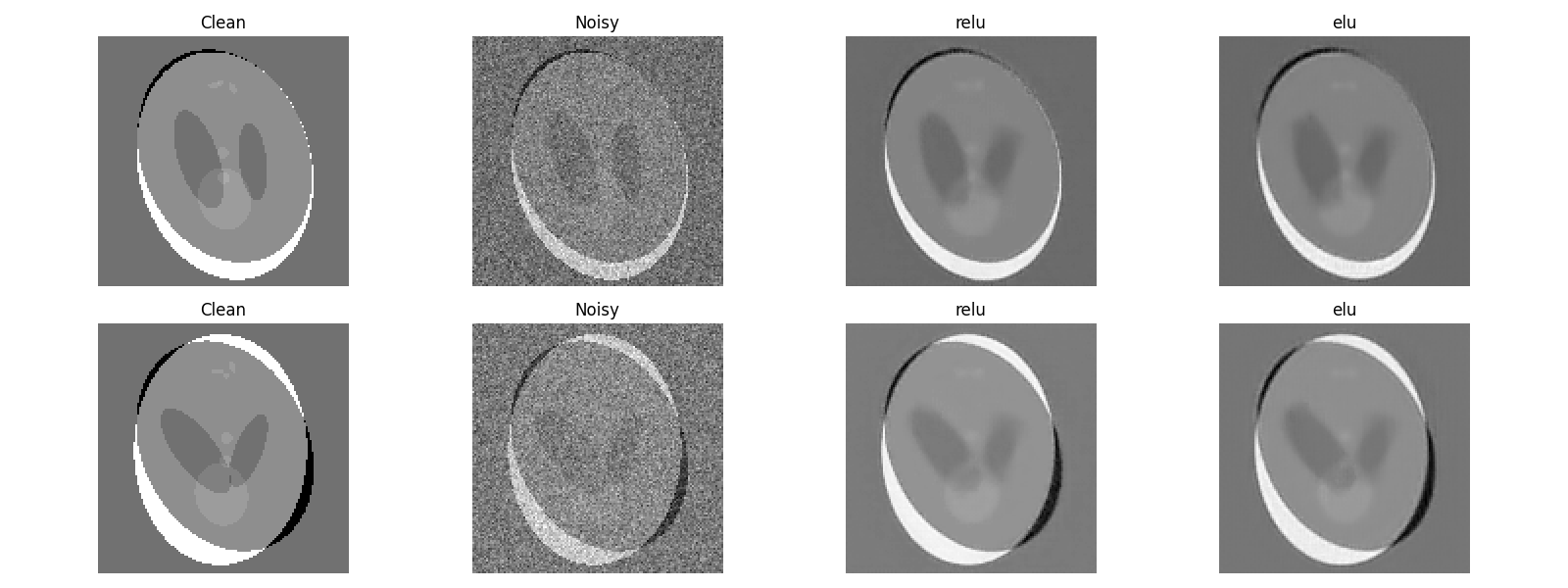}
    \caption{Class 1 with noise level $\epsilon = 1.$}
    \label{fig: easy-1}
\end{figure}

\begin{center}
\begin{tabular}{lccc} 
\toprule
\textbf{Activation} & $\sigma_{\min}(J) \uparrow$ & \textbf{$\sigma_{\max}(J) \downarrow$} & $\text{Lip. const.} \downarrow$ \\ \midrule\textbf{ReLU} & 2.052 & 13.036 & 119.675 \\ \textbf{ELU} & \bf 2.149 & \bf 11.976 & \bf 113.976 \\ \bottomrule
\end{tabular}
\captionof{table}{Class 1 with noise level $\epsilon = 5:$ stability evaluation.}
\label{table: easy-5} 
\end{center}

\begin{figure}[!htb]
    \centering
\includegraphics[width=0.9\linewidth]{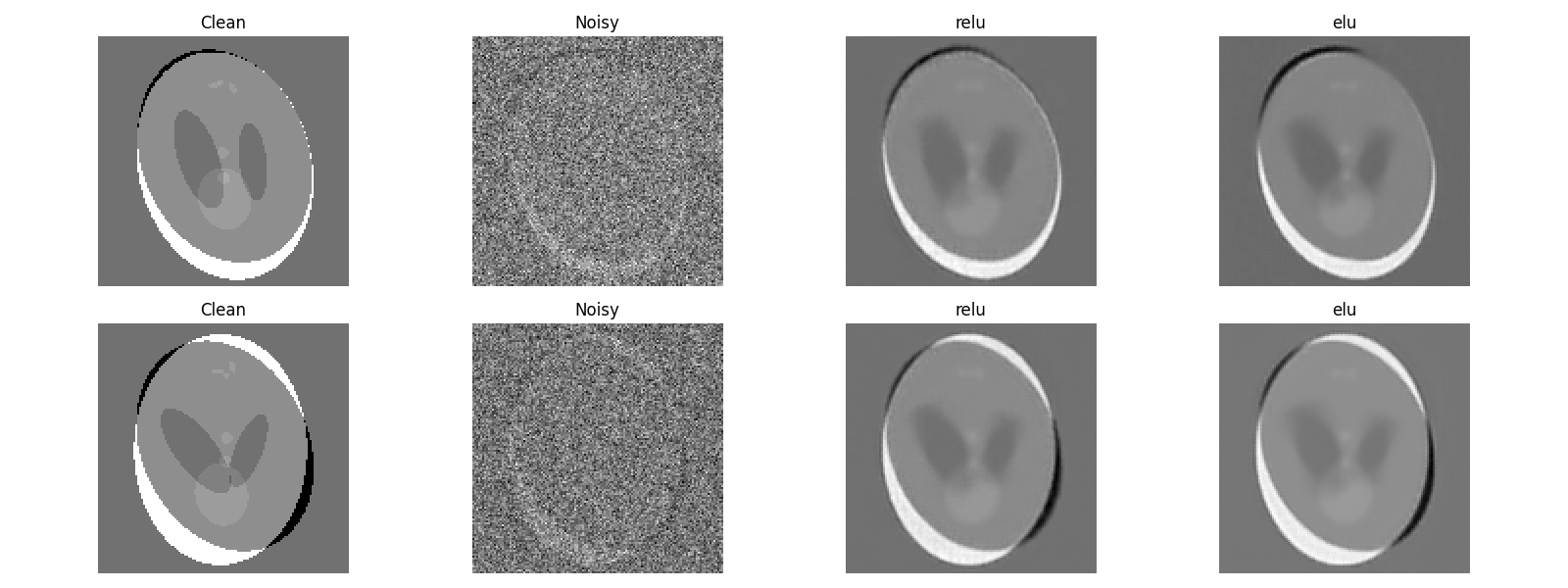}
    \caption{Class 1 with noise level $\epsilon = 5:$ denoising performance.}
    \label{fig: easy-5}
\end{figure}

\begin{center}
\begin{tabular}{lccc}
\toprule
\textbf{Activation} & \textbf{$\sigma_{\min}(J) \uparrow$} & \textbf{$\sigma_{\max}(J) \downarrow$} & $\text{Lip. const.} \downarrow$ \\
\midrule
\textbf{ReLU} & \bf 1.684 & 9.165 & 86.105 \\
\textbf{ELU} & 1.496 & \bf 8.767 & \bf 80.455 \\
\bottomrule
\end{tabular}
\captionof{table}{Class 2 with noise level $\epsilon = 5:$ stability evaluation.}
\label{table: hard-5} 
\end{center}

\begin{figure}[!htb]
    \centering
\includegraphics[width=0.9\linewidth]{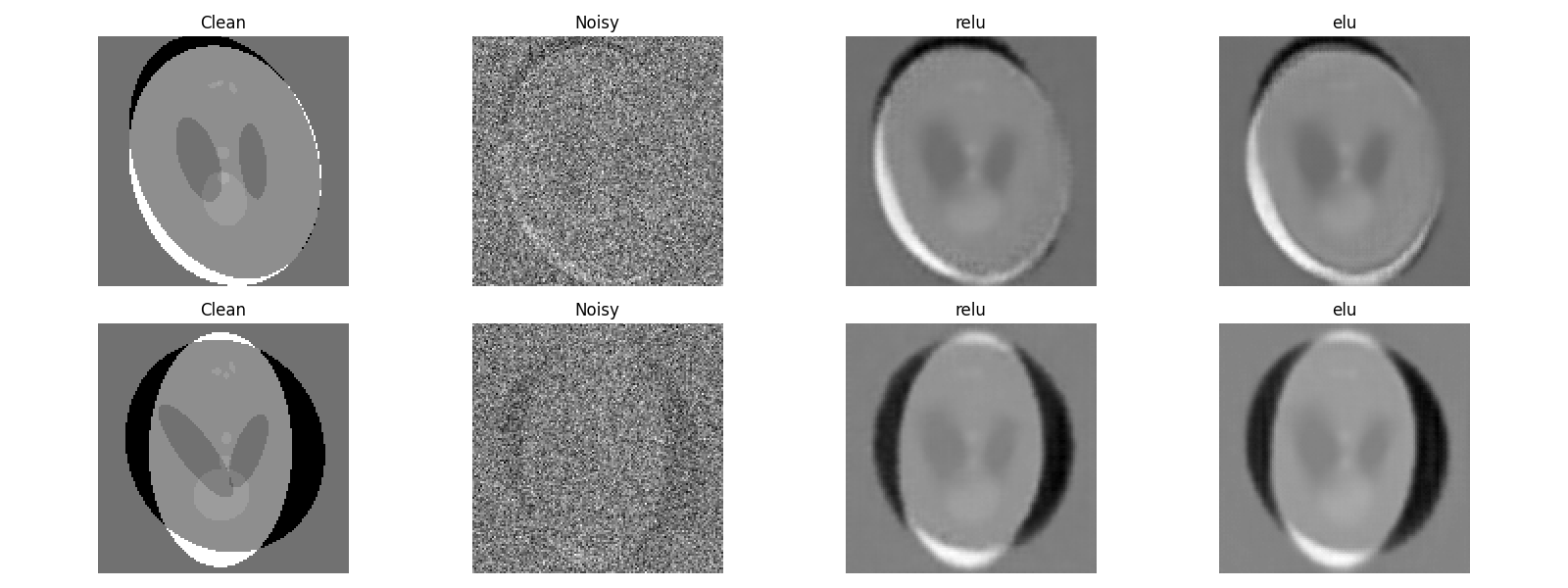}
    \caption{Class 2 with noise level $\epsilon = 5.$}
    \label{fig: hard-5}
\end{figure}

\section{Discussion and open problems} \label{sec:discussion}

The theoretical framework we propose potentially applies to other inverse problems for which deep generative priors may be obtained, like phase retrieval~(\cite{candes2015phase}) or multi-reference alignment~(\cite{bandeira2014multireference}). It is reasonable to suspect that the sample complexity of this problems may significantly decrease by the use of generative priors. Such a result would be particularly significant for the multi-reference alignment problem (on some regimes estimation of the signal is not possible unless the number of samples exceeds $1/\text{SNR}^3$)~\cite{perry2017sample}.   

A different direction to explore is whether it is possible to use this framework to study classification problems. A classification problem (with $n$ classes) can be thought as a function $c: D \to [n]$ where $D\subset S^{n}$. An interesting question is what classification functions $c$ can be approximated by using functions $L^{(\ell)}(x):=L_{n_\ell}\circ \ldots \circ L_{n_2}\circ L_{n}(x)\in \mathscr{L}^2(S^{n_{\ell}})$ where $\ell$ is the number of layers. In this framework the classifier would be approximated by a function $\bar c(x)=\arg\max_{i\in[n]} \{\langle L^{(\ell)}(x), y_i \rangle\}_{i \in [n]}$ for some $y_1, \ldots, y_n \in \mathscr{L}^2(S^{n_{\ell}})$ and $y_i$ are the objects we may want to find using local methods. We believe an answer to a problem of this form may involve the study of the geometric or topological properties of $L^{(\ell)}(S^{n})\subset \mathscr{L}^2(S^{n_{\ell}})$.

Finally, an intriguing question that arises from this analysis is what the condition $g'_\theta(t)>0$ means for the activation function $\theta$. For instance, squaring the coefficients of the Fourier decomposition of a function corresponds with convolving the function with itself in the time domain. Is there an interesting interpretation of squaring the coefficients of the Gegenbauer decomposition?

\section*{Acknowledgements} The authors would like to thank Joan Bruna, Tim Carson, Ben Recht and Ludwig Schmidt. This work was done in part while SV was visiting the Simons Institute for the Theory of Computing. DGM was partially supported by AFOSR F4FGA06060J007 and AFOSR Young Investigator Research Program award F4FGA06088J001. 

\begin{appendices}
\section{Proof of Lemma~\ref{lemma.abs.convergence}} \label{app.convergence}
The proof leverages the standard sufficient condition for convergence of Proposition~\ref{prop.convergence} and the identities~\eqref{bound.inf} and \eqref{bound.derivative}.

\begin{proposition} \label{prop.convergence}
If for all $t\in I \subset \mathbb R$ we have
\begin{enumerate}
\item[(a)] $\lim_{n\to \infty} S_n(t) = f(t)$ 
\item[(b)] $\lim_{n\to \infty} S_n'(t) = g(t)$
\item[(c)] $|S_n''(t)|<C$ for some $C$ independent of $n$  
\end{enumerate}
then $g(t)=f'(t)$ for all $t\in I \subset \mathbb R$.
\end{proposition}

\begin{proof}
Fix $t\in I$ and $h\neq 0$ such that $t+h\in I$. Then
\begin{align*}
|f(t+h)-f(t)-hg(t)|
&\le |f(t+h)-S_n(t+h)| + |f(t)-S_n(t)| + h|g(t)-S_n'(t)|\\
&\quad  + |S_n(t+h)-S_n(t)-hS_n'(t)|.
\end{align*}

By Taylor’s theorem, there exists $\xi_{n,h}$ between $t$ and $t+h$ such that
\[
S_n(t+h)-S_n(t)-hS_n'(t)
= \frac{h^2}{2} S_n''(\xi_{n,h}),
\]
and by assumption (c),
\[
|S_n(t+h)-S_n(t)-hS_n'(t)| \le \frac{C}{2}h^2.
\]

Dividing by $|h|$ gives
\[
\left|\frac{f(t+h)-f(t)}{h}-g(t)\right|
\le \left\vert\frac{f(t+h)-S_n(t+h)}{h}\right\vert
+ \left\vert\frac{f(t)-S_n(t)}{h}\right\vert
+ |g(t)-S_n'(t)|
+ \frac{C}{2}|h|.
\]

Let $n\to\infty$. By (a) and (b), the first three terms tend to zero, hence
\[
\limsup_{n\to\infty}
\left|\frac{f(t+h)-f(t)}{h}-g(t)\right|
\le \frac{C}{2}|h|.
\]

Now letting $h\to 0$ yields
\[
\lim_{h\to 0} \frac{f(t+h)-f(t)}{h} = g(t).
\]
Therefore $f$ is differentiable at all $t \in I$ and $f'(t)=g(t)$.
\end{proof}

\begin{proof}[Proof of Lemma~\ref{lemma.abs.convergence}]

Let $\omega\in S^n$ and define the function $h:S^n\to \mathbb R$ as $h(\tau)=\theta(\omega\cdot \tau)$. Then $h\in \mathscr{L}^2(S^{n})$ and $h$ is $C^2$. We have $h(\tau)=\sum_{k=0}^\infty a_k \varphi_{n,k}(\omega\cdot \tau)$, in particular $h_k=a_kF_k(\omega,\cdot)\in \mathcal H_k(S^n)$. Then \eqref{c2.bound} implies $k^2\|a_k\varphi_{n,k}(t)\|_{\mathscr{L}^2(\mu_n)} < A_n$ for some constant $A_n$ that depends on $n$ but it does not depend on $k$. Using \eqref{bound.norm} we obtain 
\begin{equation} \label{bound.ak}
a_k^2< \frac{B_n}{k^{n+3}},
\end{equation}
which implies
$$\sum_{k=K}^\infty a_k^2 \varphi_{n,k}(t) < \sum_{k=K}^\infty \frac{1}{k^{n+3}} \varphi_{n,k}(t) \leq B_n \sum_{k=K}^\infty \frac{ k^{n-1}}{k^{n+3}} < \infty  $$
which establishes the pointwise convergence of $\sum_{k=0}^\infty a_k^2 \varphi_{n,k}(t)$ to a function $g_\theta(t)$ for $n\geq 1$. Now we consider the derivatives and we use the identity \eqref{bound.derivative} and we get
$$\sum_{k=0}^\infty a_k^2 \varphi_{n,k}'(t) = C_n \sum_{k=0}^\infty a_k^2 \varphi_{n+2,k-1}(t) \leq D_n \sum_{k=0}^\infty a_k^2 k^{n+1} $$
where $C_n, D_n$ are constants depending only on $n$. Note that this argument guarantees the pointwise convergence of $\sum_{k=0}^\infty a_k^2 \varphi_{n,k}'(t) < \infty$. 

Using \eqref{bound.derivative} again we get
$$\sum_{k=0}^\infty a_k^2 \varphi_{n,k}''(t) = E_n \sum_{k=0}^\infty a_k^2 \varphi_{n+2,k-1}'(t) = F_n \sum_{k=0}^\infty a_k^2 \varphi_{n+4,k-2}(t) \leq G_n \sum_{k=0}^\infty a_k^2 k^{n+3}. $$
Now the bound \eqref{bound.ak} is not good enough to bound the second derivative, but if we have that $\theta$ is $C^4$ we can use \eqref{c2.bound} with $r=2$ obtaining a bound that allows us to use Proposition~\ref{prop.convergence} and complete the proof of Lemma~\ref{lemma.abs.convergence}.
\end{proof}

\section{Derivation of bounds in Table \ref{table}} \label{app.bounds}
Let $\theta:[-1,1]\to \mathbb R$ be a $C^4$ activation function. Let $h:S^n\to \mathbb R$ be a function defined as $h(\tau)=\theta(\tau\cdot\omega)$ for some $\omega\in S^n$ fixed. 
Then bound \ref{c2.bound} says that $k^4\|h_k\|\leq \|\Delta_S^2 h \|$.

Similar computations than the ones in the proof of Lemma \ref{lemma.theta} show that 
$\|h_k\|^2=a_k^2\frac{\alpha_{n,k}}{\operatorname{vol(S^n)}}$ and the norm of the laplacian can be computed as 
$$\|\Delta_S^2 h \|^2= \operatorname{vol}(S^{n-1})\int_{-1}^1 (\Delta^2_{S^n}(\theta(t)))^2(1-t^2)^{\frac{n-2}{2}}dt=: \operatorname{vol}(S^{n-1})\Delta_{\theta,n},$$
where $\Delta_{\theta,n}$ is a constant depending only on $\theta$ and $n$ that we can compute for each activation function.
We have $$a_n^2 \leq \frac{\Delta_{\theta,n}\operatorname{vol}(S^{n-1})}{\operatorname{vol}(S^{n}) \alpha_{n,k} k^8}.$$
Then using \eqref{bound.derivative} and \eqref{bound.inf} we have 
$$|a_k^2\varphi_{n,k}'(t)|\leq \frac{\Delta_{\theta,n}\operatorname{vol}(S^{n-1})}{\operatorname{vol}(S^{n}) \alpha_{n,k} k^8} \frac{(n+1)\operatorname{vol}(S^n)}{\operatorname{vol}(S^{n+2})}|\varphi_{n+2,k-1}(t)| = \frac{\Delta_{\theta,n}\operatorname{vol}(S^{n-1})}{ \alpha_{n,k} k^8} \frac{(n+1)}{\operatorname{vol}(S^{n+2})} \frac{\alpha_{n+2,k-1}}{\operatorname{vol}(S^{n+2})}.$$
Note that $\frac{\alpha_{n+2,k-1}}{\alpha_{n,k}}= \frac{(k+n-1)k}{n+2}$, obtaining $$|a_k^2\varphi_{n,k}'(t)|<A_{\theta,n} \frac{1}{k^6}.$$
One can uniformly bound the tail $\displaystyle\sum_{k=K+1}^\infty |a_k^2\varphi_{n,k}'(t)|$ by observing that 
$\displaystyle\sum_{k=K+1}^\infty \frac{1}{k^6}\leq \int_{K}^\infty \frac{1}{t^6} dt = \frac{K^{-5}}5$.

\end{appendices}

\bibliography{references}
\end{document}